\documentclass[11pt]{article}

\usepackage[letterpaper,margin=1in]{geometry}
\usepackage{times}
\usepackage{microtype}
\usepackage{natbib}
\usepackage[hypertexnames=false]{hyperref}
\usepackage{url}

\usepackage{amsmath,amsfonts,bm}

\def\eqref#1{equation~\ref{#1}}
\def\Eqref#1{Equation~\ref{#1}}

\def\1{\bm{1}}

\DeclareMathAlphabet{\mathsfit}{\encodingdefault}{\sfdefault}{m}{sl}
\SetMathAlphabet{\mathsfit}{bold}{\encodingdefault}{\sfdefault}{bx}{n}

\newcommand{\E}{\mathbb{E}}

\newcommand{\R}{\mathbb{R}}

\DeclareMathOperator*{\argmin}{arg\,min}

\usepackage{amsthm,amssymb}
\usepackage{graphicx}
\usepackage{float}
\usepackage{xcolor}
\usepackage{tikz}
\usetikzlibrary{arrows.meta,calc,positioning,patterns}

\newtheorem{theorem}{Theorem}[section]
\newtheorem{proposition}[theorem]{Proposition}
\newtheorem{lemma}[theorem]{Lemma}
\newtheorem{corollary}[theorem]{Corollary}
\theoremstyle{definition}
\newtheorem{definition}[theorem]{Definition}
\newtheorem{assumption}[theorem]{Assumption}
\newtheorem{example}[theorem]{Example}
\theoremstyle{remark}
\newtheorem{remark}[theorem]{Remark}

\DeclareMathOperator{\supp}{supp}

\DeclareMathOperator{\relint}{relint}
\DeclareMathOperator{\spanop}{span}

\newcommand{\Sset}{S}
\newcommand{\zstar}{z^*}
\newcommand{\Wstar}{W^*}
\newcommand{\wstar}{w^*}
\newcommand{\supportfun}{\xi_{\Sset}}
\newcommand{\cbar}{\bar c}
\newcommand{\RiskSPO}{R_{\mathrm{SPO}}}
\newcommand{\RiskSPOP}{R_{\mathrm{SPO+}}}
\newcommand{\lossSPO}{\ell_{\mathrm{SPO}}}
\newcommand{\lossSPOP}{\ell_{\mathrm{SPO+}}}
\newcommand{\crossgap}[1]{G_{#1}}
\newcommand{\Leb}{\lambda}
\newcommand{\TieLocus}{\mathcal{T}}
\newcommand{\SymDiff}{\mathbin{\triangle}}

\definecolor{PaperBlue}{HTML}{0072B2}
\definecolor{PaperOrange}{HTML}{D55E00}
\definecolor{PaperGreen}{HTML}{009E73}
\definecolor{PaperGray}{HTML}{666666}
\definecolor{PaperInk}{HTML}{222222}

\hypersetup{
  colorlinks=true,
  linkcolor=blue,
  citecolor=blue,
  urlcolor=blue,
  pdftitle={Geometric Identification in Predict-Then-Optimize Learning},
  pdfauthor={Jiaxiao Xu; Changhong Mou; Keji Liu; Dinghua Xu; Yeyu Zhang}
}

\title{Geometric Identification in Predict-Then-Optimize Learning}
\author{Jiaxiao Xu$^{1}$ \quad Changhong Mou$^{2}$ \quad Keji Liu$^{1}$\\[0.4ex]
  Dinghua Xu$^{1}$ \quad Yeyu Zhang$^{1}$\\[0.6ex]
  $^{1}$Shanghai University of Finance and Economics\\
  $^{2}$Utah State University}
\date{}

\begin{document}
\maketitle

\begin{abstract}
Decision-focused surrogates can recover downstream decisions without
identifying the quotient report. We characterize the equality set of the
convex Smart Predict-then-Optimize surrogate (SPO+) population risk. Under
central symmetry, the centered mean class is the unique
Bayes minimizer exactly when every nonzero effective displacement makes the old
optimizer leave the shifted optimal face with positive probability. This
condition separates face crossing from selected-oracle disagreement and gives
quantitative local coercivity. Without symmetry, strict crossing alone need not
identify the mean; selection balance with reflected crossing restores
quotient-report identification, and conditional versions extend the result to
measurable predictors. These are population statements, without finite-sample
report-recovery or generic transfer-regret guarantees. Closed-form mechanisms
reproduce the analytic identities and rates. Portfolio, complete-matrix
KuaiRec, and Energy/Storage studies measure predictive fidelity, shifted
regret, and fitted-report geometry. A known data-generating process (DGP)
companion retains their application geometries while isolating conditional-mean
recovery and crossing, without testing the original observational assumptions.
\end{abstract}

\section{Introduction}
\label{sec:introduction}

Predict-then-optimize learning uses a predictive report as the cost vector of a
downstream optimization problem, often trained by decision-focused methods
\citep{DontiAmosKolter2017,MandiEtAl2022}. The Smart Predict-then-Optimize
(SPO) loss and its convex SPO+ surrogate
\citep{ElmachtoubGrigas2022} are commonly assessed through decision quality,
with calibration and generalization analyses addressing when small surrogate
risk yields good decisions \citep{LiuGrigas2021,ElBalghitiEtAl2023}. Decision
quality alone does not determine the predictive report. If multiple reports
induce the same downstream action or remain in the same optimality region, they
can have indistinguishable training-task decisions while encoding different
cost predictions.
Decision consistency is therefore weaker than quotient-report identification:
small task regret does not identify the predictive report, even modulo invisible
directions.

The distinction becomes operational when a fitted report is reused. A cost vector
that is decision-equivalent for one feasible set need not remain equivalent
after an objective, constraint, or feasible-set change. The resulting transfer
and reuse implication is only a testable hypothesis: similar training-task
regret need not imply similar quotient-report fidelity or shifted-task
performance. This paper establishes the population geometry; it does not prove
a transfer benefit. The real-data studies evaluate a pre-specified,
application-specific transfer hypothesis.

A one-dimensional example already shows the nonuniqueness. In
Example~\ref{ex:compact-support},
$\Sset=[-1/2,1/2]$ and a symmetric absolutely continuous cost law supported on
$[1,2]$ yield $\argmin_{\hat c}\RiskSPOP(\hat c)=[1,\infty)$, although every
point in this ray induces the Bayes-optimal decision. Thus neither a density
nor correct downstream decisions identify the quotient report. Our question is:
under which geometric conditions does SPO+ identify the effective component of
a predictive cost, namely the quotient report $[c]$, rather than only a
downstream-optimal action?

For a deterministic measurable oracle $\wstar(c)\in\Wstar(c)$ and a report
displacement $\Delta$, consider the old decision evaluated at the shifted cost,
\[
 \crossgap{\Delta}(c)=(c+2\Delta)^\top\wstar(c)-\zstar(c+2\Delta)\geq0.
\]
The gap is positive exactly when the old selected decision leaves the entire
shifted optimal face. Prior SPO+ analyses give the centered excess-risk identity
and decision consistency or mean minimization under stronger conditions
\citep{LiuGrigas2021}. We do not claim another SPO+ consistency theorem:
Theorem~\ref{thm:identification-sc} identifies optimal-face strict crossing as
the exact condition for unique quotient-report identification, characterizing
the zero-excess-risk geometry after invisible directions are quotiented out.
Subsequent results separate face departure from oracle disagreement and give
conditions for local coercivity.

Central symmetry centers the Bayes target at $\E[C]=\cbar$ and yields the
nonnegative centered excess-risk identity; it is not required to define
optimal-face strict crossing, establish geometric crossing arguments, or
compare crossing with oracle disagreement under absolute continuity. Without
symmetry, optimal-face strict crossing (SC) alone does not guarantee mean
identification:
Theorem~\ref{thm:symmetry-free-obstruction} constructs a full-support,
mean-zero asymmetric law for which SC holds but the mean is not SPO+-optimal.
Theorem~\ref{thm:selection-balanced-identification} gives exact quotient-report
identification when the oracle balances under reflection and reflected
crossing holds; approximate symmetry instead localizes the Bayes set under a
coercivity condition.

The fixed-covariate theory also has a contextual lift.
Theorem~\ref{thm:contextual-identification} treats $C\mid X=x$ as a Borel
probability kernel and optimizes over a measurable predictor class. Conditional
central symmetry and strict crossing give uniqueness of the conditional-mean
predictor after quotient projection, up to almost-sure equality; the quantitative
corollary integrates a uniform conditional coercivity modulus. These are
population Bayes statements and do not imply automatic finite-sample
quotient-report recovery, neural-network optimization convergence, or improved
shifted-task regret without additional assumptions.

The contributions are:
\begin{itemize}
  \item \textbf{Exact zero-excess-risk geometry.} Theorem~\ref{thm:identification-sc}
  characterizes the exact condition under which the complete population
  minimizer set is one quotient-report class: every nonzero effective
  displacement must produce optimal-face departure with positive probability.
  Theorem~\ref{thm:full-support} gives a geometric sufficient condition.
  \item \textbf{Oracle/crossing separation and quantitative coercivity.}
  Selected-oracle disagreement can occur within a shared optimal face; absolute
  continuity removes this ambiguity for each fixed displacement
  (Theorem~\ref{thm:ac-equivalence}). Quantitative crossing yields local
  parameter coercivity, including quadratic growth for direction-covering
  polyhedral boundary mass and sharp one-dimensional exponents
  (Section~\ref{sec:quantitative-crossing}).
  \item \textbf{A calibrated symmetry boundary.}
  Theorem~\ref{thm:symmetry-free-obstruction} rules out symmetry-free SC
  identification of the mean. Theorem~\ref{thm:selection-balanced-identification}
  recovers exact quotient-report identification under oracle selection balance and reflected
  crossing; total-variation asymmetry yields Bayes-set localization.
  \item \textbf{Contextual population identification.}
  Theorem~\ref{thm:contextual-identification} and
  Corollary~\ref{cor:contextual-quantitative-identification} lift the fixed-$x$
  result to measurable predictors under explicit conditional assumptions.
\end{itemize}

The next section distinguishes this equality-set question from regret analyses,
decision-sufficiency methods, and local analyses near decision boundaries.

Controlled synthetic experiments reproduce crossing, oracle separation, and
coercivity. Portfolio, complete-matrix KuaiRec, and Energy/Storage measure
predictive fidelity, pre-specified-shift regret, and fitted-report crossing.
Their observational outcomes do not identify a conditional Bayes report or
establish population strict crossing. We therefore add a known-DGP layer that
keeps each application's feature rows and optimization geometry fixed while
setting the conditional cost law by construction. This layer tests whether
conditional-mean recovery, crossing mass, and finite-family transfer regret
follow the predicted mechanism; it does not establish a universal transfer
guarantee. The visualizations are reported in Appendix~\ref{app:semisynthetic-experiments}.

\section{Related work}
\label{sec:related-work}

\paragraph{Decision-focused learning with SPO+.}
\citet{ElmachtoubGrigas2022} introduced SPO and its convex surrogate SPO+.
Existing SPO+ theory gives the centered excess-risk identity and establishes
decision consistency or unique mean minimization under stronger symmetry,
density, and full-dimensionality conditions; later calibration, margin, and
generalization results relate excess risk to task regret
\citep{LiuGrigas2021,ElBalghitiEtAl2023,HoNguyenKilincKarzan2022};
\citet{ImBenslimaneGrigas2025} extends this target to robust constraints.
\citet{WanLiu2026} proposes a solver-free surrogate for predict-then-optimize
training and proves Fisher consistency and excess-risk bounds for that loss.
These results concern decision consistency or sufficient conditions for a target;
they do not characterize the complete zero-excess-risk equality set of SPO+.
Our contribution is not another SPO+ consistency theorem: under centered
symmetry, we characterize the population minimizer geometry after quotienting
invisible directions, with optimal-face strict crossing as the necessary and
sufficient condition for a single quotient-report class.

\paragraph{What optimal decisions identify.}
Recent work characterizes data or deterministic proxies sufficient for an
optimal decision, sometimes allowing nonunique optimal point forecasts, while
inverse optimization characterizes objective parameters compatible with
observed decisions
\citep{BennounaAminOzdaglar2025,YeAminOzdaglar2026,SchutteEtAl2025,HomemDeMelloEtAl2026,AhmadiEtAl2026}.
\citet{SunLiuLi2023} builds a max-margin objective from optimality conditions
and trains from observed optimal solutions without objective coefficients; it
does not study the population SPO+ equality set.
Related work on elicitable functionals and calibrated structured surrogates
distinguishes the statistical target, the report, and the regret induced by a
surrogate \citep{Gneiting2011,FisslerHlavinovaRudloff2021,FrongilloWaggoner2021,FinocchiaroFrongilloWaggoner2024,OsokinBachLacosteJulien2017,NowakVilaBachRudi2019}.
Those are information-, proxy-, or observation-identification targets. Ours is
the Bayes equality set of a fixed surrogate: when does zero SPO+ excess risk
identify one effective component, i.e. one quotient report, rather than merely
preserve a downstream decision?

\paragraph{Geometry near decision boundaries.}
Normal-fan margins and contextual fast-rate analyses examine behavior near
decision boundaries \citep{ElBalghitiEtAl2023,HuKallusMao2022}. In contrast,
we use optimal-face crossing to characterize the global equality set of the
population SPO+ risk, and use direction-covering boundary mass to obtain local
coercivity. This separates disagreement of a selected oracle from departure
from the full shifted optimal face.

\section{Problem Setup and Population Notions}
\label{sec:setup}

Sections~\ref{sec:strict-crossing} through~\ref{sec:quantitative-crossing} work at a
fixed covariate value and suppress conditioning notation; Appendix~\ref{sec:extensions}
states the measurable contextual lift. Let
$\Sset\subset\R^d$ be nonempty, compact, and convex. For a cost vector
$q\in\R^d$, define
\begin{equation}
  \zstar(q):=\min_{w\in\Sset}q^\top w,
  \qquad
  \Wstar(q):=\argmin_{w\in\Sset}q^\top w.
  \label{eq:optimal-value-and-face}
\end{equation}
The optimizer correspondence $\Wstar$ is set-valued. Throughout, $\wstar$ is
a deterministic measurable selection satisfying $\wstar(q)\in\Wstar(q)$ for
every $q$. The support function of $\Sset$ is
$\supportfun(q):=\max_{w\in\Sset}q^\top w=-\zstar(-q)$.

For a prediction $\hat c$ and observed cost $c$, the SPO+ loss is
\begin{equation}
  \lossSPOP(\hat c,c)
  :=\supportfun(c-2\hat c)+2\hat c^\top\wstar(c)-\zstar(c).
  \label{eq:spo-plus-loss}
\end{equation}
Write $\RiskSPOP(\hat c):=\E[\lossSPOP(\hat c,c)]$ whenever the expectation
is finite, and let $P_c$ denote the law of $c$. To make the downstream target
explicit, we deploy the same fixed
selection and define
\begin{equation}
  \lossSPO(\hat c,c):=c^\top\wstar(\hat c)-\zstar(c),
  \qquad
  \RiskSPO(\hat c):=\E[\lossSPO(\hat c,c)].
  \label{eq:deployed-spo-risk}
\end{equation}
The target parameter is $\cbar\in\R^d$. Let
\[
  L:=\operatorname{span}(\Sset-\Sset),\qquad N:=L^\perp,\qquad
  [q]:=q+L^\perp\in\R^d/L^\perp.
\]
Let $P_L$ denote the orthogonal projection onto $L$.
Write $k:=\dim L$ and let $\Leb_L$ denote $k$-dimensional Lebesgue measure
on $L$.
Since every $n\in N=L^\perp$ is constant on $\Sset$, the SPO+ risk is invariant under
$\hat c\mapsto\hat c+n$. Accordingly, every identification statement below
is a statement about the quotient report $[\hat c]=\hat c+L^\perp$, equivalently
the effective projected report $P_L\hat c$. When $\Sset$ is full-dimensional,
$N=\{0\}$ and the quotient coincides with the ambient parameter space.

\begin{assumption}[Centered distribution and integrability]
\label{assump:centered-symmetric}
The cost law is centrally symmetric about $\cbar$, namely
$c\overset{d}=2\cbar-c$, and $\E\|c\|<\infty$. The deterministic measurable
selection $\wstar$ is fixed before population risks are formed.
\end{assumption}

Compactness makes $\wstar(c)$ bounded, so Assumption~\ref{assump:centered-symmetric}
is sufficient for the integrability used below. It also implies
$\cbar=\E[c]$. For coordinatewise nonnegative costs, symmetry further
forces $0\leq c_j\leq2\cbar_j$ almost surely for each coordinate $j$. In particular,
\begin{equation}
  \RiskSPO(\hat c)=\cbar^\top\wstar(\hat c)-\E[\zstar(c)],
  \qquad
  \argmin_{\hat c}\RiskSPO(\hat c)
  =\{\hat c:\wstar(\hat c)\in\Wstar(\cbar)\}.
  \label{eq:spo-bayes-set-fixed-selection}
\end{equation}
Thus this task-risk notion is explicit about the deployed tie-breaking rule.

\begin{proposition}[Centered Bayes minimizer]
\label{prop:bayes-minimizer}
Under Assumption~\ref{assump:centered-symmetric}, $\RiskSPOP$ is convex and
$\cbar\in\argmin_{\hat c}\RiskSPOP(\hat c)$.
\end{proposition}

We distinguish the following population notions.  They are often conflated in
decision-focused learning, and the examples in this paper separate them.

\begin{enumerate}
  \item \textbf{Bayes existence}: $\cbar\in\argmin\RiskSPOP$.
  \item \textbf{Unique quotient-report identification}:\\
  \[\argmin\RiskSPOP=\cbar+N \iff
  \argmin_{[\hat c]}\RiskSPOP([\hat c])=\{[\cbar]\}.\]
  \item \textbf{Fixed-rule Fisher consistency}: every Bayes report induces a
  Bayes-optimal decision for the deployed selection:
  $\argmin\RiskSPOP\subseteq\argmin\RiskSPO$, with the latter set given
  exactly by \eqref{eq:spo-bayes-set-fixed-selection}.
  \item \textbf{Selection-independent decision consistency}:\\
  \[\Wstar(\hat c)\subseteq\Wstar(\cbar) \quad \text{for every}
  \quad\hat c\in\argmin\RiskSPOP\text{.}\]
  \item \textbf{Quantitative quotient-report identification}: there is a modulus
  $\psi$ such that, for effective displacements $\Delta\in L$,
  $\RiskSPOP(\cbar+\Delta)-\RiskSPOP(\cbar)\geq\psi(\|\Delta\|)$ locally.
  This is quotient-parameter coercivity, not task-regret calibration.
  \item \textbf{Estimator recovery}: an empirical minimizer near $\cbar$ has
  effective-component error controlled by an empirical uniform-deviation and
  the local inverse modulus.  This requires an additional localization
  argument and is not implied by the population coercivity alone.
\end{enumerate}

\begin{remark}[Terminology]
We use ``unique quotient-report identification'' to mean uniqueness of the
quotient Bayes report $[\hat c]$, never uniqueness of an ambient representative. In
the terminology of scoring-function theory this is closely related to strict
consistency and should not be confused with identification-function
terminology for statistical functionals
\citep{Gneiting2011,FisslerHlavinovaRudloff2021}.  The present paper is not a
scoring-rule survey; the distinction matters because a strictly consistent
loss can have a unique quotient Bayes report without implying any particular downstream
decision rule, and conversely decision consistency can hold without unique
identification.
\end{remark}

\section{Absolute Continuity Does Not Ensure Identification}
\label{sec:continuity-failure}

\begin{example}[Compact support yields a flat SPO+ risk]
\label{ex:compact-support}
Absolute continuity does not by itself make a decision-focused surrogate
identify the quotient report. Let $\Sset=[-1/2,1/2]$ and
$C\sim\operatorname{Unif}[1,2]$, so $\cbar=3/2$. Direct evaluation gives
\begin{equation}
  \argmin_{\hat c}\RiskSPOP(\hat c)=[1,\infty),
  \label{eq:compact-support-plateau}
\end{equation}
although every minimizer induces the same singleton decision as the mean.
Indeed, for $c>0$,
\begin{equation}
  \lossSPOP(\hat c,c)=\tfrac12|c-2\hat c|-\hat c+\tfrac c2,
  \label{eq:compact-support-loss}
\end{equation}
so the loss vanishes almost surely for $\hat c\geq1$ and is positive with
positive probability for $\hat c<1$. The exact piecewise risk and the same
construction with a smooth compactly supported density are given in
Appendix~\ref{app:counterexamples}.
\end{example}

The obstruction is mass missing from the optimization-transition regions, not
the lack of a density. Under the centered-symmetry setup, unique quotient-report
identification instead requires
positive mass in every nonzero effective displacement's optimal-face crossing
event, formalized by optimal-face strict crossing in Section~\ref{sec:strict-crossing} and
quantified in Section~\ref{sec:quantitative-crossing}.

\begin{figure*}[t]
  \centering
  \resizebox{\textwidth}{!}{\begin{tikzpicture}[x=0.86cm,y=0.86cm,>=Latex,font=\small]
  \begin{scope}
    \node[font=\bfseries\small,anchor=south west,text=PaperInk] at (-0.35,3.45)
      {\normalfont\footnotesize\bfseries (a) same decision region};
    \draw[fill=black!1.5,rounded corners=3pt] (-0.35,-2.30) rectangle (6.90,3.10);
    \draw[->,PaperGray,line width=1.0pt] (0,0) -- (6.15,0) node[right] {$q$};
    \draw[PaperGray,thick] (3.2,-0.30) -- (3.2,1.85);
    \node[below left,font=\footnotesize,inner sep=1.5pt] at (3.2,0) {$0$};
    \node[above,font=\footnotesize,align=center,inner sep=2pt] at (0.95,1.00) {$q<0$:\\$w=+B$};
    \node[above,font=\footnotesize,align=center,inner sep=2pt] at (5.45,1.00) {$q>0$:\\$w=-B$};
    \fill[PaperBlue!20] (4.25,0.04) rectangle (5.75,0.36);
    \draw[PaperBlue,very thick,rounded corners=1pt] (4.25,0.20) -- (5.75,0.20);
    \node[PaperBlue,above,font=\footnotesize,inner sep=2pt] at (5.0,0.42) {$\operatorname{supp}(P_c)$};
    \draw[->,PaperOrange,very thick] (4.6,-0.78) -- (5.35,-0.78);
    \node[PaperOrange,below,font=\footnotesize,inner sep=2pt] at (4.98,-0.94) {$2\Delta$};
    \node[align=center,font=\footnotesize,text=PaperGray,inner sep=3pt] at (3.2,-1.72)
      {No crossing:\\$G_\Delta(c)=0$ a.s.\ $\Rightarrow$ flat risk direction};
  \end{scope}

  \begin{scope}[xshift=9.15cm]
    \node[font=\bfseries\small,anchor=south west,text=PaperInk] at (-0.35,3.45)
      {\normalfont\footnotesize\bfseries (b) mass crosses a decision boundary};
    \draw[fill=black!1.5,rounded corners=3pt] (-0.35,-2.30) rectangle (6.90,3.10);
    \draw[->,PaperGray,line width=1.0pt] (0,0) -- (6.15,0) node[right] {$q$};
    \draw[PaperGray,thick] (3.2,-0.45) -- (3.2,1.85);
    \node[below left,font=\footnotesize,inner sep=1.5pt] at (3.2,0) {$0$};
    \node[above,font=\footnotesize,align=center,inner sep=2pt] at (0.95,1.00) {$q<0$:\\$w=+B$};
    \node[above,font=\footnotesize,align=center,inner sep=2pt] at (5.45,1.00) {$q>0$:\\$w=-B$};
    \fill[PaperOrange!22] (2.05,0) rectangle (3.2,0.50);
    \draw[PaperOrange,thick,rounded corners=1pt] (2.05,0.50) -- (3.2,0.50);
    \node[PaperOrange,align=center,font=\footnotesize,inner sep=2pt] at (1.80,2.42) {crossing strip\\$-2\Delta<c<0$};
    \draw[PaperOrange,thin] (2.35,1.94) -- (2.60,0.58);
    \draw[->,PaperOrange,very thick] (2.4,-0.78) -- (3.65,-0.78);
    \node[PaperOrange,below,font=\footnotesize,inner sep=2pt] at (2.72,-0.94) {$2\Delta$};
    \node[align=center,font=\footnotesize,text=PaperGray,inner sep=3pt] at (3.2,-1.72)
      {Positive-mass crossing:\\$\Pr\{G_\Delta(c)>0\}>0$};
  \end{scope}
\end{tikzpicture}}
  \caption{Missing transition-region mass leaves a flat surrogate-risk
  direction; mass in a crossed boundary strip yields a positive gap.}
  \label{fig:central-phenomenon}
\end{figure*}

\section{Strict Crossing and Exact Equality-Set Geometry}
\label{sec:strict-crossing}

For a displacement $\Delta\in\R^d$, define the \emph{optimization crossing
gap}
\begin{equation}
  \crossgap{\Delta}(c)
  :=(c+2\Delta)^\top\wstar(c)-\zstar(c+2\Delta)\geq0.
  \label{def:crossing-gap}
\end{equation}
It is the suboptimality of the old selected decision at the shifted cost.
The nonnegativity in \eqref{def:crossing-gap} is pointwise and the equality
event has a set-valued meaning:
\begin{equation}
  \crossgap{\Delta}(c)>0
  \quad\Longleftrightarrow\quad
  \wstar(c)\notin\Wstar(c+2\Delta).
  \label{eq:gap-positive-iff-leaves-face}
\end{equation}

\begin{definition}[Optimal-face strict crossing]
\label{def:strict-crossing}
The triple $(P_c,\Sset,\wstar)$ satisfies \emph{optimal-face strict crossing}
(SC) when,
for every effective displacement $\Delta\in L\setminus\{0\}$,
\[
  \Pr\{\wstar(c)\notin\Wstar(c+2\Delta)\}>0.
  \qquad\textup{(SC)}
\]
\end{definition}

\begin{theorem}[Identification if and only if strict crossing]
\label{thm:identification-sc}
Under Assumption~\ref{assump:centered-symmetric},
\begin{equation}
  \cbar+N\text{ is the unique minimizer set of }\RiskSPOP
  \quad\Longleftrightarrow\quad
  \text{SC holds.}
  \label{eq:identification-iff-sc}
\end{equation}
\end{theorem}

\begin{proof}[Proof idea]
The centered excess-risk identity is
\begin{equation}
  \RiskSPOP(\cbar+\Delta)-\RiskSPOP(\cbar)
  =\E\left[\crossgap{\Delta}(c)\right].
  \label{eq:spo-plus-excess-identity}
\end{equation}
By \eqref{eq:gap-positive-iff-leaves-face}, the expectation in
\eqref{eq:spo-plus-excess-identity} is strictly positive exactly when the SC
event has positive probability. Proposition~\ref{prop:bayes-minimizer} then
gives both directions of \eqref{eq:identification-iff-sc}. Appendix~\ref{app:proofs}
records the centered algebra and the full quantifier argument.
\end{proof}

\begin{figure*}[t]
  \centering
  \includegraphics[width=0.98\textwidth]{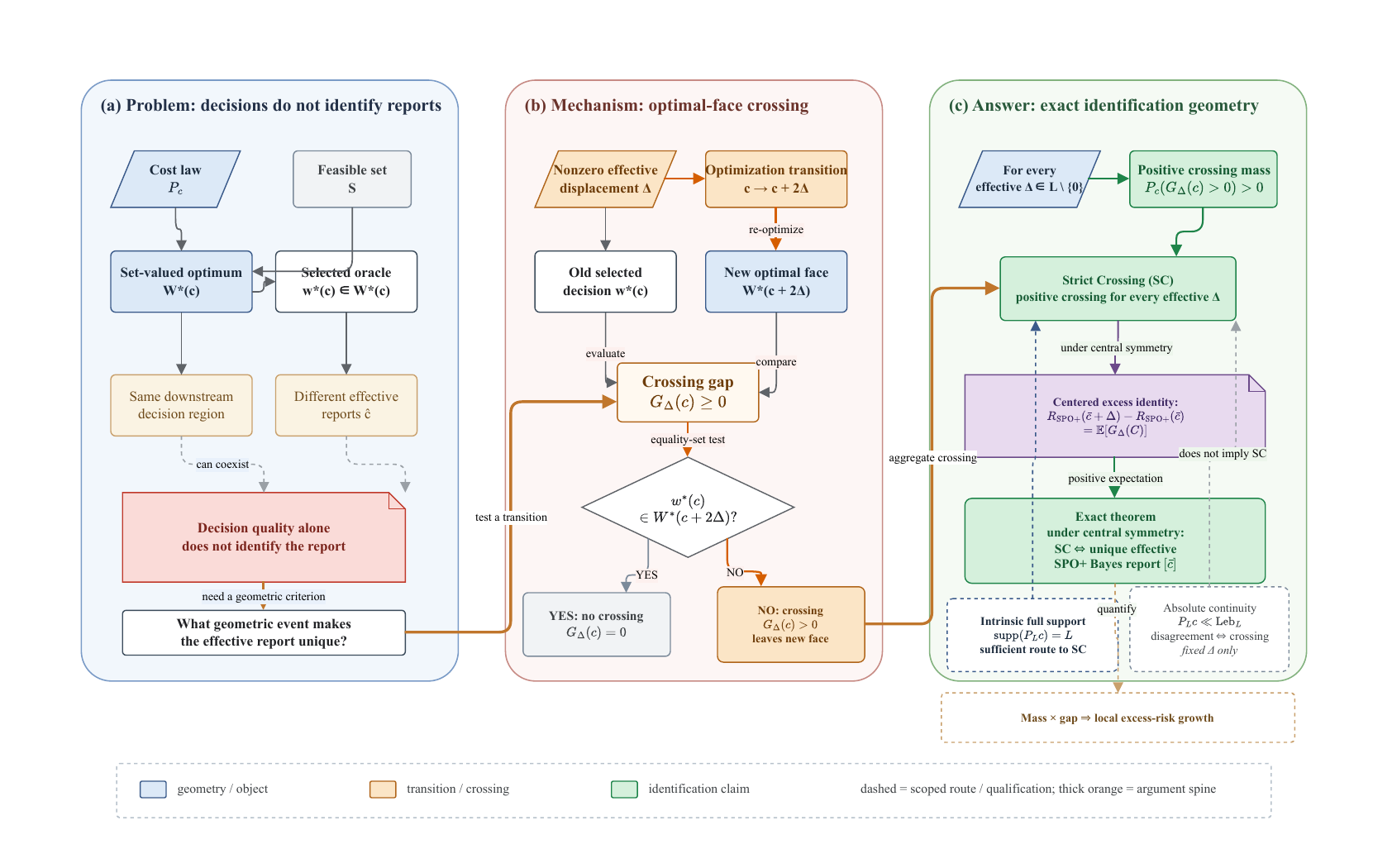}
  \caption{Optimal-face strict crossing links optimal-face departure to exact quotient
  identification; the side conditions and quantitative refinement are shown in
  the diagram.}
  \label{fig:strict-crossing-flow}
\end{figure*}

\begin{remark}[What is known and what is new here]
The centered representation \eqref{eq:spo-plus-excess-identity} and uniqueness
under stronger full-support conditions are known \citep{LiuGrigas2021}.
Theorem~\ref{thm:identification-sc} characterizes their exact equality case
after quotienting invisible directions: zero gap means that the old decision
remains in the shifted optimal face. The distinct geometric and quantitative
consequences are developed below.
\end{remark}

Theorem~\ref{thm:identification-sc} treats a decision transition as the
primitive object and identifies only the effective quotient report. It does
not compare the two selected oracle outputs; that distinction is essential on
a positive-probability tie set.

\section{Geometry of Strict Crossing}
\label{sec:geometry}

The support function makes the set-valued geometry behind
Theorem~\ref{thm:identification-sc} explicit.  All geometry in this section
is intrinsic to $L$: translating $\Sset$ by any $s_0\in\Sset$ gives the
full-dimensional body $\Sset-s_0\subset L$ without changing its optimizers.
Let
\begin{equation}
  \TieLocus_L:=\{q\in L:\Wstar(q)\text{ is not a singleton}\}.
  \label{eq:tie-locus}
\end{equation}
The restricted support function is finite convex on $L$, so $\TieLocus_L$ is
$\Leb_L$-null. For a polytope, $q\mapsto\Wstar(q)$ is constant on the
relative interiors of the minimization normal cones in $L$.

\begin{theorem}[Intrinsic full support implies strict crossing]
\label{thm:full-support}
Suppose $\supp(P_Lc)=L$. Then SC in Definition~\ref{def:strict-crossing}
holds.
\end{theorem}

\begin{proof}[Proof idea]
Fix $\Delta\in L\setminus\{0\}$ and set $q_0=-\Delta$. Every
$w\in\Wstar(q_0)$ maximizes $\Delta^\top w$, whereas
$\zstar(q_0+2\Delta)=\zstar(\Delta)$. Hence
\begin{equation}
  \Delta^\top w-\zstar(\Delta)
  =\max_{u\in\Sset}\Delta^\top u-\min_{u\in\Sset}\Delta^\top u>0.
  \label{eq:directional-width}
\end{equation}
The width is positive because a vector in $L$ that is orthogonal to every
element of $\Sset-\Sset$ must vanish. Put
\[
  g_\Delta(q,w):=(q+2\Delta)^\top w-\zstar(q+2\Delta).
\]
Since $\zstar$ is continuous, so is $g_\Delta$. At $q=q_0$, uniformly over
$w\in\Wstar(q_0)$,
\[
  g_\Delta(q_0,w)
  =\max_{u\in\Sset}\Delta^\top u-\min_{u\in\Sset}\Delta^\top u
  =:\delta_\Delta>0.
\]
Compactness and continuity give relative neighborhoods $U_1$ of $q_0$ in $L$
and $V$ of
$\Wstar(q_0)$ on which $g_\Delta>\delta_\Delta/2$. By the maximum theorem,
$q\mapsto\Wstar(q)$ is upper hemicontinuous, so some relative neighborhood
$U_2$ of $q_0$ satisfies $\Wstar(q)\subset V$. Thus, for $U=U_1\cap U_2$,
\[
  \Wstar(q)\subset V,\qquad
  g_\Delta(q,w)>\delta_\Delta/2\quad\forall w\in\Wstar(q),
\]
which implies $\wstar(q)\notin\Wstar(q+2\Delta)$. Hence the SC event contains
the preimage of the nonempty relatively open set $U$, and intrinsic full
support gives $\Pr\{P_Lc\in U\}>0$. The appendix gives the reduction and
full argument.
\end{proof}

The proof constructs a crossing neighborhood before invoking the distribution.
It explains why intrinsic full support is sufficient and why a density inside
one decision region, as in Example~\ref{ex:compact-support}, is not.

\begin{remark}[Quotient-report identification is intrinsic]
\label{rem:quotient-space-identification}
Let $L:=\spanop(\Sset-\Sset)$ and $L^\perp$ be its orthogonal complement.
For any $\Delta\in L^\perp$, the linear form $\Delta^\top w$ is constant on
$\Sset$, and a pointwise calculation gives
$\lossSPOP(\hat c+\Delta,c)=\lossSPOP(\hat c,c)$ for every $\hat c,c$
(Appendix~\ref{app:geometry}).  Hence the ambient SPO+ risk is constant along
$L^\perp$ for \emph{every} distribution, including asymmetric laws.
The identification target is therefore the quotient report
$[\cbar]=\cbar+L^\perp$, or equivalently the effective component
$P_L\cbar$.
Theorem~\ref{thm:identification-sc} accordingly quantifies SC over
$L\setminus\{0\}$ and identifies the quotient report in
$\R^d/L^\perp$; no ambient full-dimensionality is required.
\end{remark}

\begin{figure*}[t]
  \centering
  \resizebox{\textwidth}{!}{\begin{tikzpicture}[x=0.82cm,y=0.82cm,>=Latex,font=\footnotesize]
  \begin{scope}
    \node[font=\bfseries\small,anchor=south west,text=PaperInk] at (-0.55,3.65)
      {\normalfont\footnotesize\bfseries (a) normal-fan crossing};
    \draw[fill=black!1.5,draw=PaperGray!35,line width=0.5pt,rounded corners=3pt]
      (-0.55,-1.70) rectangle (4.95,3.30);
    \coordinate (o) at (2.3,0.15);
    \fill[PaperBlue!10] (o) -- (0.1,2.7) -- (2.3,2.25) -- cycle;
    \fill[PaperGreen!12] (o) -- (2.3,2.25) -- (4.6,2.7) -- cycle;
    \fill[PaperOrange!12] (o) -- (4.6,2.7) -- (2.3,-0.45) -- cycle;
    \draw[PaperBlue,thick] (o) -- (0.1,2.7);
    \draw[PaperGray,thick] (o) -- (2.3,2.25);
    \draw[PaperGreen,thick] (o) -- (4.6,2.7);
    \draw[PaperOrange,thick] (o) -- (2.3,-0.45);
    \node[PaperBlue,font=\footnotesize,inner sep=2pt] at (1.05,1.70) {$C_{v_1}$};
    \node[PaperGreen,font=\footnotesize,inner sep=2pt] at (3.70,1.70) {$C_{v_2}$};
    \node[PaperOrange,font=\footnotesize,inner sep=2pt] at (3.55,0.55) {$C_{v_3}$};
    \draw[fill=black] (1.75,1.30) circle (1.5pt);
    \node[below left,font=\footnotesize,inner sep=2pt] at (1.75,1.30) {$c$};
    \draw[->,PaperOrange,very thick] (1.75,1.30) -- (3.10,1.55)
      node[midway,above,font=\footnotesize,inner sep=2pt] {$2\Delta$};
    \draw[fill=black] (3.10,1.55) circle (1.5pt);
    \node[above right,font=\footnotesize,inner sep=2pt] at (3.40,1.10) {$c{+}2\Delta$};
    \node[align=center,font=\footnotesize,text=PaperGray,inner sep=3pt] at (2.3,-1.25)
      {$c\in C_{v_1}$, $c{+}2\Delta\in C_{v_2}$:\\$v_1\notin\Wstar(c{+}2\Delta)$};
  \end{scope}

  \begin{scope}[xshift=8.0cm]
    \node[font=\bfseries\small,anchor=south west,text=PaperInk] at (-0.2,3.65)
      {\normalfont\footnotesize\bfseries (b) assumption hierarchy};
    \draw[fill=black!1.5,draw=PaperGray!35,line width=0.5pt,rounded corners=4pt]
      (-1.75,-2.35) rectangle (10.15,3.35);
    \node[draw=PaperBlue,rounded corners=4pt,align=center,
          fill=PaperBlue!8,minimum width=5.2cm,minimum height=0.85cm,
          line width=0.9pt,inner sep=4pt] (fs) at (1.9,2.4)
      {intrinsic full support\\$\operatorname{supp}(P_Lc)=L$};
    \node[draw=PaperGreen,rounded corners=4pt,align=center,
          fill=PaperGreen!10,minimum width=3.8cm,minimum height=0.85cm,
          line width=0.9pt,inner sep=4pt] (sc) at (1.9,1.15)
      {strict crossing};
    \node[draw=PaperOrange,rounded corners=4pt,align=center,
          fill=PaperOrange!10,minimum width=5.0cm,minimum height=1.05cm,
          line width=0.9pt,inner sep=4pt] (id) at (1.9,-0.15)
      {under central symmetry\\unique quotient-report identification};
    \draw[->,thick,PaperBlue] (fs) -- (sc);
    \draw[->,thick,PaperGreen] (sc) -- (id);
    \node[draw=PaperOrange!60,rounded corners=3pt,align=center,
          fill=PaperOrange!4,minimum width=5.6cm,minimum height=0.9cm,
          line width=0.7pt,font=\scriptsize,inner sep=4pt] (cx) at (1.9,-1.45)
      {C1: density with compact support $\not\Rightarrow$ SC\\
       C5: atom-free $\not\Rightarrow$ equivalence};
    \node[draw=PaperGray,rounded corners=4pt,align=center,
          fill=gray!6,minimum width=4.0cm,minimum height=1.7cm,
          line width=0.7pt,font=\footnotesize,inner sep=5pt,dashed] (ac) at (7.55,1.10)
      {absolute continuity\\$P_Lc\ll\Leb_L$\\[1pt]
       fixed $\Delta$: disagreement\\$\Leftrightarrow$ crossing a.s.};
  \end{scope}
\end{tikzpicture}}
  \caption{Normal-fan crossing and the assumption hierarchy. Intrinsic full support
  supplies crossing mass; under central symmetry, strict crossing gives unique
  quotient-report identification. Absolute continuity only resolves oracle ambiguity.}
  \label{fig:normal-fan-hierarchy}
\end{figure*}

\section{Oracle Disagreement and Optimal-Face Crossing}
\label{sec:oracle-disagreement}

For a fixed displacement define
\begin{equation}
  A_\Delta:=\{\wstar(c)\neq\wstar(c+2\Delta)\},\qquad
  B_\Delta:=\{\wstar(c)\notin\Wstar(c+2\Delta)\}.
  \label{eq:disagreement-and-crossing-events}
\end{equation}
Always $B_\Delta\subseteq A_\Delta$, but the converse can fail on a
non-singleton shifted optimal face.
\begin{example}[Atomic oracle disagreement without crossing]
\label{ex:atomic-disagreement}
The atomic construction in Appendix~\ref{app:counterexamples} has
$\Pr(A_\Delta)=1/2$ and $\Pr(B_\Delta)=0$.
\end{example}
\begin{example}[Atom-free disagreement without crossing]
\label{ex:cantor-separation}
The singular Cantor construction in Appendix~\ref{app:counterexamples} is
atom-free but still has $\Pr(A_\Delta)=1/2$ and $\Pr(B_\Delta)=0$.
\end{example}
In both cases the oracle moves tangentially within a common optimal face while
the SPO+ crossing gap remains zero. Thus atom-freeness is insufficient.

Absolute continuity removes this ambiguity displacement-wise:
\begin{theorem}[Absolute continuity removes output ambiguity]
\label{thm:ac-equivalence}
Suppose $P_Lc\ll\Leb_L$. For every fixed $\Delta\in\R^d$,
\begin{equation}
  \Pr(A_\Delta\mathbin{\triangle}B_\Delta)=0.
  \label{eq:ac-disagreement-crossing-equivalence}
\end{equation}
\end{theorem}
\begin{proof}[Proof idea]
The projected tie locus in $L$ is $\Leb_L$-null because the restricted support
function is finite convex. Hence the source and shifted projected costs avoid
it almost surely, and selected-output disagreement is equivalent to departure
from the shifted singleton optimal face. The statement is for each fixed
$\Delta$; the full argument is in Appendix~\ref{app:geometry}.
\end{proof}

\section{Quantitative Crossing and Excess-Risk Growth}
\label{sec:quantitative-crossing}

Strict crossing certifies positivity but does not supply a modulus. The
selected gap always obeys a linear envelope: if
$D_{\Sset}:=\sup_{u,v\in\Sset}\|u-v\|$, then
\begin{equation}
  0\leq\crossgap{\Delta}(c)\leq2D_{\Sset}\|\Delta\|.
  \label{eq:universal-gap-envelope}
\end{equation}
We therefore control both the mass of a crossing event and the magnitude of
the gap on that event.

\begin{definition}[Quantitative strict crossing (QSC)]
\label{def:quantitative-margin}
For $r,\gamma,\kappa>0$, $\alpha\geq1$, and $\beta\geq0$,
QSC$(\alpha,\beta;\gamma,\kappa,r)$ holds if, for every effective
displacement $\Delta\in L$ with $0<\|\Delta\|\leq r$,
\begin{equation}
  \Pr\{\crossgap{\Delta}(c)\geq\gamma\|\Delta\|^\alpha\}
  \geq\kappa\|\Delta\|^\beta.
  \label{eq:quantitative-strict-crossing}
\end{equation}
\end{definition}

\begin{theorem}[Quantitative crossing transfer]
\label{thm:quantitative-crossing}
Under Assumption~\ref{assump:centered-symmetric}, QSC$(\alpha,\beta;
\gamma,\kappa,r)$ implies
\begin{equation}
  \RiskSPOP(\cbar+\Delta)-\RiskSPOP(\cbar)
  \geq\gamma\kappa\|\Delta\|^{\alpha+\beta}
  \quad\text{for }\Delta\in L,\ 0<\|\Delta\|\leq r.
  \label{eq:qsc-risk-transfer}
\end{equation}
\end{theorem}

\begin{proposition}[Local polynomial coercivity implies quantitative crossing]
\label{prop:coercivity-qsc}
Under Assumption~\ref{assump:centered-symmetric}, assume the effective space is
nontrivial, $\dim L\geq1$ (equivalently $D_{\Sset}>0$), and suppose there are
$C,r>0$ and $p\geq1$ such that
\begin{equation}
  \RiskSPOP(\cbar+\Delta)-\RiskSPOP(\cbar)
  \geq C\|\Delta\|^p
  \quad\text{for every }\Delta\in L,\ 0<\|\Delta\|\leq r.
  \label{eq:local-coercivity-hypothesis}
\end{equation}
Then QSC$(p,p-1;\tfrac C2,\tfrac{C}{4D_{\Sset}},r)$ holds.
\end{proposition}

\begin{proof}[Proof]
See Appendix~\ref{app:quantitative-proofs}.
\end{proof}

\begin{remark}[Quantitative crossing and coercivity are two views of one local phenomenon]
\label{rem:qsc-coercivity-equivalence}
Theorem~\ref{thm:quantitative-crossing} and
Proposition~\ref{prop:coercivity-qsc} form a two-way relationship at the
level of local polynomial moduli, up to transformed exponents and constants:
QSC$(\alpha,\beta)$ yields growth exponent $p=\alpha+\beta$, while a
growth exponent $p\geq1$ yields QSC$(p,p-1)$.  The universal linear envelope
\eqref{eq:universal-gap-envelope} forces $p\geq1$ for any nontrivial
polynomial lower bound.  Strict crossing is the qualitative limit: it asserts
$\E[\crossgap{\Delta}]>0$ for every nonzero $\Delta$, but no modulus.
QSC strengthens this by jointly controlling crossing mass and crossing gap on
the effective quotient directions.
\end{remark}

For polytopes, an intrinsic normal-fan condition makes the two exponents
explicit. Let $\Sset$ be a compact polytope with affine hull $s_0+L$. For a
vertex $v$, let $C_v:=\{q\in L:v\in\Wstar(q)\}$ be its minimization cone.
For each edge
$e=[v_i,v_j]$, let $F_e$ be the codimension-one face shared by $C_{v_i}$ and
$C_{v_j}$, and set $a_e=v_i-v_j$ and
$L_e:=\spanop(F_e)=L\cap a_e^\perp$. Here
$L_e$ is the linear span of the fan facet, not its lineality space.

\begin{lemma}[Edge directions span the effective space]
\label{lem:edge-directions-span}
The directions of the edges of $\Sset$ span $L$. Consequently, for any finite
family $\mathcal E$ whose directions span $L$,
\begin{equation}
  \tau_{\mathcal E}
  :=\min_{u\in L:\|u\|=1}\max_{e\in\mathcal E}|a_e^\top u|>0.
  \label{eq:direction-covering-constant}
\end{equation}
\end{lemma}

\begin{proof}[Proof]
See Appendix~\ref{app:quantitative-proofs}.
\end{proof}

The following assumption strengthens Theorem~\ref{thm:full-support} into a
quantitative statement.  It does \emph{not} require every edge of the polytope
to have a density-positive patch: only a finite family whose directions cover
all unit directions is needed.  This is the direction-covering form of the
boundary-mass condition.

\begin{assumption}[Intrinsic polyhedral boundary mass]
\label{assump:polyhedral-boundary-mass}
Let $k=\dim L\geq1$. Let $\mathcal E_0$ be a finite family of edges of
$\Sset$ whose directions span $L$; equivalently $\tau_{\mathcal E_0}>0$ in
\eqref{eq:direction-covering-constant}. The law of $P_Lc$ has density $p$
with respect to $\Leb_L$. For
every $e\in\mathcal E_0$, there are a compact patch
$K_e\Subset\relint(F_e)$ of positive $(k-1)$-dimensional measure, a radius
$\rho_e>0$, and a common constant $m>0$.  With
\begin{equation}
  K_e^{[\rho_e]}:=\{y+h:y\in K_e,\ h\in L_e,
  \ \|h\|\leq\rho_e\},
  \label{eq:tangential-facet-neighborhood}
\end{equation}
assume $K_e^{[\rho_e]}\subset\relint(F_e)$ and, for every
$y\in K_e^{[\rho_e]}$,
\begin{align}
  y+\frac{s}{\|a_e\|^2}a_e&\in\operatorname{int}(C_{v_i})
    &&\text{if }-\rho_e<s<0, \\
  y+\frac{s}{\|a_e\|^2}a_e&\in\operatorname{int}(C_{v_j})
    &&\text{if }0<s<\rho_e.
  \label{eq:two-cone-facet-property}
\end{align}
Finally, $p$ admits a Borel version satisfying, for
$(\mathcal H^{k-1}\!\restriction_{K_e})\otimes\Leb_1$-almost every
$(y,s)\in K_e\times[-\rho_e,\rho_e]$,
\begin{equation}
  p\left(y+\frac{s}{\|a_e\|^2}a_e\right)\geq m.
  \label{eq:polyhedral-boundary-mass}
\end{equation}
\end{assumption}

\begin{theorem}[Direction-covering boundary mass yields quadratic growth]
\label{thm:polyhedral-quadratic-growth}
Let $\Sset$ be a compact polytope. Under
Assumption~\ref{assump:centered-symmetric} and
Assumption~\ref{assump:polyhedral-boundary-mass}, there exist $C,r>0$ such
that
\begin{equation}
  \RiskSPOP(\cbar+\Delta)-\RiskSPOP(\cbar)
  \geq C\|\Delta\|^2
  \quad\text{for every }\Delta\in L,\ 0<\|\Delta\|\leq r.
  \label{eq:polyhedral-quadratic-growth}
\end{equation}
\end{theorem}

All cones, interiors, densities, and Hausdorff measures in this condition are
taken in $L$. Thus the result applies directly to equality-constrained
polytopes without imposing an ambient density on the cost law.

\begin{proof}[Mechanism]
By Lemma~\ref{lem:edge-directions-span}, the finite family in
Assumption~\ref{assump:polyhedral-boundary-mass} has a positive covering
constant $\tau_{\mathcal E_0}$.  For any unit direction, at least one family
edge has uniformly positive inner product with it. A displacement of size $t$
then sweeps a prism of width proportional to $t$ across the corresponding fan
facet. The density floor assigns this prism probability proportional to $t$,
and its trimmed interior has crossing gap proportional to $t$. The product
proves \eqref{eq:polyhedral-quadratic-growth}. Appendix~\ref{app:quantitative-proofs}
gives the uniform finite-edge argument with explicit constants.
\end{proof}

The growth bound is an inverse quotient-report identification modulus, not
task-regret transfer; Appendix~\ref{app:quantitative-proofs} gives the
conditional approximate-ERM consequence. It is population-only and does not
automatically imply finite-sample quotient-report recovery, neural-network
optimization convergence, or improved transfer regret without additional
assumptions.

For $\Sset=[-B,B]$ and even $p(c)\asymp|c|^\nu$ near the boundary
($\nu>-1$), the exact rate is
\begin{equation}
  \RiskSPOP(t)-\RiskSPOP(0)=\Theta\bigl(|t|^{\nu+2}\bigr).
  \label{eq:one-dimensional-sharp-rate}
\end{equation}
A positive boundary density gives a quadratic rate; a vanishing density gives a
flatter polynomial rate, and SC alone need not give a polynomial modulus.

\paragraph{Beyond central symmetry.}
For asymmetric laws, SC alone can fail to identify the mean; selection balance
with reflected crossing restores exact quotient-report identification
(Appendix~\ref{app:extensions}).
\begin{theorem}[Selection-balanced exact identification]
\label{thm:selection-balanced-identification}
Let $\mu=\E[C]$, $C^\circ=2\mu-C$, and $\E\|C\|<\infty$. If
$\E[\wstar(C)]=\E[\wstar(C^\circ)]$, then
$\argmin_\theta R_{P,\mathrm{SPO+}}(\theta)=\mu+N$ if and only if
$\Pr\{\wstar(C^\circ)\notin\Wstar(C^\circ+2\Delta)\}>0$ for every
$\Delta\in L\setminus\{0\}$.
\end{theorem}
The selection-balance condition compares the average deployed oracle under
the observed cost law with its reflected-cost counterpart. It is automatic
under central symmetry and can be plausible when the cost law and tie-breaking
rule have a reflection-compatible aggregate, but it is not generic for
asymmetric laws. If the cost law, its mean, and the oracle are known, the two
bounded expectations can be estimated or computed; for ordinary observational
data, the reflected counterfactual and exact equality are not directly
observable, so this remains a population condition rather than a routine
empirical test.
Under central symmetry this is Theorem~\ref{thm:identification-sc}; reflection
preserves the same crossing geometry.

\section{Real-Data Diagnostics}
\label{sec:real-experiments}

Observational data cannot test Theorem~\ref{thm:identification-sc}: they do not
reveal conditional means, central symmetry, or population strict crossing.
Appendix~\ref{app:semisynthetic-experiments} retains application features and
geometries under a known cost law. RMSE and fixed-shift regret are not theorem
or transfer tests. Repeated outcomes assess mean recovery under centered noise;
a deterministic law controls flat equality. Finite-grid curves illustrate,
but do not establish, assumptions or universal transfer.

\paragraph{Applications and protocol.}
Portfolio raises a 48-industry cap ($0.10$ to $0.12$); KuaiRec uses
user-disjoint complete matrices ($k\in\{5,10,15\}$); Energy/Storage uses 48
prices, an intercept, and tilts $s\in\{0,0.5,1\}$. Appendix~\ref{app:real-experiments}
gives five seeds and full shift matrices (Tables~\ref{tab:portfolio-complete},
\ref{tab:recommendation-complete}, \ref{tab:energy-complete}).

\paragraph{Observed patterns.}
On KuaiRec, Linear SPO+ leads formal methods at all tested budgets; on Energy,
SPO+ leads only at the largest tilt. Portfolio shifted regrets are close despite
up to sevenfold report-RMSE differences. Fitted crossing is descriptive.
Appendix~\ref{app:semisynthetic-experiments} shows improved mean recovery and
crossing on tested geometries, not population strict crossing or universal
transfer.

\section{Discussion, limitations, and conclusion}
\label{sec:discussion}

Under symmetry, $\E[C]$ is the Bayes target and the centered identity holds.
Without symmetry, selection balance with reflected crossing identifies the
quotient report; approximate symmetry localizes the Bayes set, but SC alone
may fail. These population results imply neither finite-sample recovery, neural
convergence, nor transfer regret; lower-dimensional feasible sets identify
only a quotient. Empirical checks (fixed shifts, known DGPs, no KuaiRec online
exposure, paired Energy results, finite neural grid) do not verify the
theorems' universal quantifiers.


\newpage

\section*{Reproducibility Statement}
The main text and appendices state the assumptions and quantifiers for every
result. The source bundle contains the full proofs, counterexamples,
closed-form illustrations, experiment protocols, and reported figures.

\section*{AI Use Statement}
We used generative AI tools to assist with research planning, literature
synthesis, checking mathematical derivations and proof presentation, reviewing
experimental design and implementation, interpreting experiment outputs, and
revising manuscript prose. We did not use generative AI to generate synthetic
datasets, clean or reformat datasets, or conduct qualitative or thematic data
analysis. The authors manually checked AI-assisted mathematical material against
the stated arguments, verified software changes and numerical outputs with the
reported tests, and take responsibility for the final text, claims, proofs, and
software artifacts.

\bibliographystyle{plainnat}
\bibliography{references}

\clearpage
\appendix
\section{Proofs for Centering and Strict Crossing}
\label{app:proofs}

\begin{proof}[Proof of Proposition~\ref{prop:bayes-minimizer}]
For fixed $c$, \eqref{eq:spo-plus-loss} is convex in $\hat c$, and
\begin{equation}
  g_c:=2\bigl(\wstar(c)-\wstar(2\cbar-c)\bigr)
  \in\partial_{\hat c}\lossSPOP(\cbar,c).
  \label{eq:spo-plus-subgradient}
\end{equation}
Indeed, $\partial_{\hat c}\supportfun(c-2\hat c)=-2\Wstar(2\hat c-c)$, so
$g_c$ is the sum of $-2\wstar(2\cbar-c)$ and the derivative $2\wstar(c)$ of
the report-dependent term.  By convexity,
\begin{equation}
  \lossSPOP(\hat c,c)
  \ge\lossSPOP(\cbar,c)+g_c^\top(\hat c-\cbar)
  \quad\text{for every }\hat c,c.
  \label{eq:pointwise-subgradient-inequality}
\end{equation}
Both sides are integrable: the loss is bounded by a constant multiple of
$\|c\|+\|\hat c\|+1$ because $\Sset$ is compact, and $g_c$ is bounded
for the same reason.  Taking expectations in
\eqref{eq:pointwise-subgradient-inequality} gives
\begin{equation}
  \RiskSPOP(\hat c)
  \ge\RiskSPOP(\cbar)+\E[g_c]^\top(\hat c-\cbar).
  \label{eq:integrated-subgradient-inequality}
\end{equation}
Under Assumption~\ref{assump:centered-symmetric}, $c\overset d=2\cbar-c$,
so the deterministic selections satisfy
$\E[\wstar(c)]=\E[\wstar(2\cbar-c)]$; hence $\E[g_c]=0$.  Therefore
\eqref{eq:integrated-subgradient-inequality} shows
$\RiskSPOP(\hat c)\ge\RiskSPOP(\cbar)$ for every $\hat c$, proving
$\cbar\in\argmin_{\hat c}\RiskSPOP(\hat c)$.  Convexity of
$\RiskSPOP$ follows from convexity of \eqref{eq:spo-plus-loss} in $\hat c$
and linearity of the expectation.
\end{proof}

\begin{lemma}[Centered excess-risk identity]
\label{lem:centered-excess-identity}
Under Assumption~\ref{assump:centered-symmetric}, for every $\Delta\in\R^d$,
\begin{equation}
  \RiskSPOP(\cbar+\Delta)-\RiskSPOP(\cbar)
  =\E[\crossgap{\Delta}(c)].
  \label{eq:appendix-centered-excess-identity}
\end{equation}
\end{lemma}

\begin{proof}
Subtract \eqref{eq:spo-plus-loss} at $\cbar$ from its value at
$\cbar+\Delta$.  Before using symmetry, the difference is
\begin{equation}
  \E\bigl[\supportfun(c-2\cbar-2\Delta)-\supportfun(c-2\cbar)
      +2\Delta^\top\wstar(c)\bigr].
  \label{eq:pre-symmetry-excess-algebra}
\end{equation}
Since $c-2\cbar\overset d=-c$, symmetry changes the expectation of the
support-function difference into the corresponding expression at $-c$; the
linear term is unchanged.  Thus
\begin{align}
  \RiskSPOP(\cbar+\Delta)-\RiskSPOP(\cbar)
  &=\E\bigl[\supportfun(-c-2\Delta)-\supportfun(-c)
      +2\Delta^\top\wstar(c)\bigr] \\
  &=\E\bigl[(c+2\Delta)^\top\wstar(c)-\zstar(c+2\Delta)\bigr].
  \label{eq:appendix-excess-algebra}
\end{align}
For the second equality, use $\supportfun(-q)=-\zstar(q)$ and
$c^\top\wstar(c)=\zstar(c)$. This is \eqref{eq:appendix-centered-excess-identity}.
\end{proof}

\begin{proof}[Proof of Theorem~\ref{thm:identification-sc}]
Lemma~\ref{lem:centered-excess-identity} writes the excess risk as the
expectation of a nonnegative random variable. By optimality at $c+2\Delta$,
the random variable vanishes exactly when
$\wstar(c)\in\Wstar(c+2\Delta)$. Therefore the excess is positive for every
effective displacement $\Delta\in L\setminus\{0\}$ exactly when SC holds.
Proposition~\ref{prop:bayes-minimizer} shows that $\cbar$ is already a global
minimizer. If SC holds, every report outside the coset $\cbar+N$ has strictly
larger risk; if SC fails for some nonzero $\Delta\in L$, the excess is zero
and $\cbar+\Delta$ is a second quotient minimizer. Reports differing only by
an element of $N$ are equivalent by the quotient invariance in
Section~\ref{sec:setup}.
\end{proof}

\section{Optimal-Set Containment for General Compact Convex Sets}
\label{app:optimal-set-containment}

This section records a foundational result for the set-valued SPO loss.  The
containment statement appears in Elmachtoub and Grigas
\citep{ElmachtoubGrigas2022}; we do not claim it as a new main theorem.  We
include a self-contained compact-convex proof because the set-valued
distinction is useful for our identification analysis, and we explicitly avoid
perturbing arbitrary face-interior points.  It is
separate from the fixed-selection deployed loss in
\eqref{eq:deployed-spo-risk}: ties are evaluated by the worst decision in the
predicted optimal face.  Define
\begin{equation}
  \ell_{\mathrm{SPO}}^{\mathrm{set}}(\hat c,c)
  :=\max_{w\in\Wstar(\hat c)}c^\top w-\zstar(c),
  \qquad
  R_{\mathrm{SPO}}^{\mathrm{set}}(\hat c)
  :=\E\!\left[\ell_{\mathrm{SPO}}^{\mathrm{set}}(\hat c,c)\right].
  \label{eq:set-valued-spo-risk}
\end{equation}

\begin{theorem}[Optimal-set containment]
\label{thm:optimal-set-containment}
Let $\Sset\subset\R^d$ be nonempty, compact, and convex, assume
$\E\|c\|<\infty$, and write $\cbar:=\E[c]$.
\begin{enumerate}
  \item If $c^*\in\argmin_{\hat c}R_{\mathrm{SPO}}^{\mathrm{set}}(\hat c)$,
  then $\Wstar(c^*)\subseteq\Wstar(\cbar)$.
  \item Conversely, if $\Wstar(c^*)$ is a singleton and
  $\Wstar(c^*)\subseteq\Wstar(\cbar)$, then
  $c^*\in\argmin_{\hat c}R_{\mathrm{SPO}}^{\mathrm{set}}(\hat c)$.
\end{enumerate}
\end{theorem}

\begin{proof}
Let $F:=\Wstar(\cbar)$.  We first prove the containment.  If $F$ is a
singleton, write $F=\{e\}$ and take $\tilde c_n=\cbar$.  Otherwise choose an
exposed point $e$ of the compact convex set $F$ relative to its affine hull.
Thus there is $d\in\R^d$ such that
\begin{equation}
  \argmin_{w\in F}d^\top w=\{e\}.
  \label{eq:exposed-point-in-mean-face}
\end{equation}
For $\epsilon_n\downarrow0$, put $\tilde c_n=\cbar+\epsilon_n d$.

We claim that $\Wstar(\tilde c_n)$ converges to $\{e\}$ in Hausdorff
distance.  Take any sequence $w_n\in\Wstar(\tilde c_n)$.  Optimality against
$e$ gives
\begin{equation}
  \epsilon_n(d^\top e-d^\top w_n)
  \geq \cbar^\top w_n-\zstar(\cbar)\geq0,
  \label{eq:exposed-perturbation-sandwich}
\end{equation}
so $d^\top w_n\leq d^\top e$.  Every convergent subsequence has a limit
$w\in\Wstar(\cbar)=F$: this follows by passing to the limit in
$\tilde c_n^\top w_n\leq\tilde c_n^\top u$ for every $u\in\Sset$.  Equation
\eqref{eq:exposed-point-in-mean-face} and the preceding inequality force
$w=e$.  Compactness then implies the asserted Hausdorff convergence.

Let $c^*$ minimize $R_{\mathrm{SPO}}^{\mathrm{set}}$.  Since the support
functions of $\Wstar(\tilde c_n)$ converge pointwise to that of $\{e\}$ and
are dominated by $M\|c\|$, where $M:=\sup_{w\in\Sset}\|w\|$, dominated
convergence yields
\begin{align}
  0
  &\leq\lim_{n\to\infty}
  \left(R_{\mathrm{SPO}}^{\mathrm{set}}(\tilde c_n)
       -R_{\mathrm{SPO}}^{\mathrm{set}}(c^*)\right) \\
  &=\cbar^\top e
    -\E\!\left[\max_{w\in\Wstar(c^*)}c^\top w\right] \\
  &\leq\cbar^\top e
    -\max_{w\in\Wstar(c^*)}\cbar^\top w
  \leq0.
  \label{eq:set-valued-containment-chain}
\end{align}
The penultimate inequality is Jensen's inequality for the support function of
$\Wstar(c^*)$; the last uses
$\cbar^\top e=\zstar(\cbar)\leq\cbar^\top w$ for every $w\in\Sset$.
Equality throughout \eqref{eq:set-valued-containment-chain} implies
$\cbar^\top w=\zstar(\cbar)$ for every $w\in\Wstar(c^*)$, proving
$\Wstar(c^*)\subseteq\Wstar(\cbar)$.

For the converse, write $\Wstar(c^*)=\{w^*\}\subseteq\Wstar(\cbar)$.  Then
\begin{equation}
  R_{\mathrm{SPO}}^{\mathrm{set}}(c^*)
  =\zstar(\cbar)-\E[\zstar(c)].
\end{equation}
For arbitrary $\hat c$ and any $w\in\Wstar(\hat c)$,
\begin{align}
  R_{\mathrm{SPO}}^{\mathrm{set}}(\hat c)
  &\geq \E[c^\top w]-\E[\zstar(c)] \\
  &=\cbar^\top w-\E[\zstar(c)]
  \geq\zstar(\cbar)-\E[\zstar(c)],
\end{align}
which proves optimality of $c^*$.
\end{proof}

\begin{example}[A face-interior point need not be exposable]
\label{ex:nonexposed-face-point}
Let
$\Sset=\operatorname{conv}\{(0,0),(1,0),(0,1)\}$ and $\cbar=0$, so
$\Wstar(\cbar)=\Sset$.  The point $\bar w=(0.4,0.3)$ lies in
$\operatorname{relint}(\Sset)$, and no cost vector $q$ satisfies
$\Wstar(q)=\{\bar w\}$.  Indeed, a singleton exposed face of a polytope is a
vertex.  This example rules out proofs that perturb the cost to make an
arbitrary point of an optimal face uniquely optimal.  Theorem
\ref{thm:optimal-set-containment} instead exposes one suitable point of the
mean-optimal face and lets the entire perturbed optimal set converge to it.
\end{example}

\begin{remark}[Why compactness and the set-valued loss matter]
Compactness supplies both optimal solutions and the integrable envelope used
in dominated convergence.  The maximum in
\eqref{eq:set-valued-spo-risk} is also essential to this theorem's interface.
For the fixed-selection risk in \eqref{eq:deployed-spo-risk}, the Bayes set is
already characterized directly by
\eqref{eq:spo-bayes-set-fixed-selection}; the two risks should not be
identified on tie predictions.
\end{remark}

\section{Geometry of Tie Loci and Full Support}
\label{app:geometry}

\begin{proposition}[Affine-hull obstruction and quotient invariance]
\label{prop:affine-hull-obstruction}
Let $L:=\spanop(\Sset-\Sset)$.  If $L^\perp$ contains a nonzero vector
$\Delta$, then for every $\hat c,c\in\R^d$,
\begin{equation}
  \lossSPOP(\hat c+\Delta,c)=\lossSPOP(\hat c,c).
  \label{eq:pointwise-affine-invariance}
\end{equation}
Consequently, $\RiskSPOP(\hat c+\Delta)=\RiskSPOP(\hat c)$ for
\emph{every} distribution satisfying the integrability needed to define the
risk, and
\begin{equation}
  \Wstar(c+2\Delta)=\Wstar(c)
  \quad\text{and}\quad
  \crossgap{\Delta}(c)=0
  \qquad\text{for every }c\in\R^d.
  \label{eq:affine-hull-invisible-direction}
\end{equation}
Thus SC and ambient-space unique identification are impossible for every
distribution.  In particular, full dimensionality is necessary for a
distribution-free sufficient condition that covers all nonzero ambient
directions.
\end{proposition}

\begin{proof}
For $\Delta\in L^\perp$, there is a constant $k=\Delta^\top w$ independent of
$w\in\Sset$.  Hence
\begin{equation}
  \supportfun(q-2\Delta)=\supportfun(q)-2k
  \quad\text{for every }q\in\R^d.
  \label{eq:support-shift-constant}
\end{equation}
Applying this with $q=c-2\hat c$ and adding the linear report shift,
\begin{align}
  \lossSPOP(\hat c+\Delta,c)
  &=\supportfun(c-2\hat c-2\Delta)
    +2(\hat c+\Delta)^\top\wstar(c)-\zstar(c) \\
  &=\supportfun(c-2\hat c)-2k+2\hat c^\top\wstar(c)+2k-\zstar(c) \\
  &=\lossSPOP(\hat c,c).
\end{align}
This proves \eqref{eq:pointwise-affine-invariance}.

To prove optimizer invariance, fix $c\in\R^d$.  For any $w,u\in\Sset$,
$\Delta\in L^\perp$ gives $\Delta^\top(w-u)=0$, hence
\begin{equation}
  (c+2\Delta)^\top(w-u)=c^\top(w-u).
  \label{eq:pairwise-ordering-shift}
\end{equation}
Thus $c$ and $c+2\Delta$ induce exactly the same pairwise objective
differences on $\Sset$, so $\Wstar(c+2\Delta)=\Wstar(c)$.  Since
$\wstar(c)\in\Wstar(c)=\Wstar(c+2\Delta)$, the crossing gap vanishes by
\eqref{eq:gap-positive-iff-leaves-face}.
\end{proof}

\begin{remark}[Identification modulo the invisible subspace]
\label{rem:quotient-proof}
Let $L=\spanop(\Sset-\Sset)$.  Proposition~\ref{prop:affine-hull-obstruction}
shows that the SPO+ population risk is constant on every coset of $L^\perp$,
so the natural parameter space is the quotient $\R^d/L^\perp$, represented by
the projected cost $P_L c$.  Under Assumption~\ref{assump:centered-symmetric},
$\cbar$ is the unique minimizer of $\RiskSPOP$ on the affine subspace
$\cbar+L$ if and only if strict crossing holds for every nonzero
$\Delta\in L$.  This is exactly Theorem~\ref{thm:identification-sc} applied
after restricting the displacement set to the directions that affect
$\Wstar$; no additional full-dimensionality is needed for the quotient
statement.
\end{remark}

The next two proofs verify the intrinsic versions stated in the main text.
They use only that $\Wstar(c)=\Wstar(P_Lc)$ and that every effective shift is
represented by its projection onto $L$.

For a compact polytope with affine hull $s_0+L$, strict crossing admits an
intrinsic normal-fan representation. Let $V$ be the vertex set and
\begin{equation}
  C_v:=\{q\in L:v\in\Wstar(q)\},\qquad v\in V,
  \label{eq:appendix-vertex-normal-cones}
\end{equation}
be the closed cones of the minimization normal fan.

\begin{proposition}[Normal-fan representation of crossing]
\label{prop:normal-fan-crossing-representation}
Suppose $\Sset$ is a compact polytope and $P_Lc\ll\Leb_L$. For every fixed
$\Delta\in\R^d$,
\begin{equation}
  \Pr(B_\Delta)
  =\Pr\!\left(
    \bigcup_{v\in V}
    \left\{P_Lc\in\operatorname{int}_L(C_v):
    P_Lc+2P_L\Delta\notin C_v\right\}
  \right).
  \label{eq:normal-fan-crossing-event}
\end{equation}
The union is disjoint up to a null set.  Thus, away from tie strata, strict
crossing is exactly positive probability mass leaving the closed normal cone
of the formerly optimal vertex; landing on a shared boundary does not count
when that vertex remains optimal.
\end{proposition}

\begin{proof}
The union of the lower-dimensional faces of the normal fan is $\Leb_L$-null.
Absolute continuity therefore implies that $P_Lc$ belongs almost surely to the
relative interior of exactly one vertex cone, say $C_v$, and hence
$\Wstar(c)=\{v\}$ and $\wstar(c)=v$. By the definition of $C_v$,
\begin{equation}
  \wstar(c)\notin\Wstar(c+2\Delta)
  \quad\Longleftrightarrow\quad
  v\notin\Wstar(c+2\Delta)
  \quad\Longleftrightarrow\quad
  P_Lc+2P_L\Delta\notin C_v.
\end{equation}
Taking the union over the almost-surely unique source cone proves
\eqref{eq:normal-fan-crossing-event}.  The interiors of distinct
full-dimensional-in-$L$ vertex cones are disjoint.
\end{proof}

\begin{remark}[Tie boundaries and oracle independence]
Proposition~\ref{prop:normal-fan-crossing-representation} removes the source
oracle from the event because absolute continuity makes the source optimizer
unique almost surely.  At the shifted cost, the closed cone $C_v$ retains the
set-valued test: if $c+2\Delta$ lands on a shared fan face that still contains
$v$, then the old decision remains optimal and no crossing occurs.  This is
the geometric distinction between optimal-face departure and an arbitrary
change of selected output.
\end{remark}

\begin{proof}[Proof of Theorem~\ref{thm:ac-equivalence}]
Fix $s_0\in\Sset$ and restrict the support function of $\Sset-s_0$ to $L$.
It is finite convex on $L$, so its nondifferentiability locus, namely
$\TieLocus_L$, is $\Leb_L$-null. Let $A_\Delta,B_\Delta$ be as in
\eqref{eq:disagreement-and-crossing-events}. The inclusion
$B_\Delta\subseteq A_\Delta$ is immediate. Outside
$\{P_Lc+2P_L\Delta\in\TieLocus_L\}$, the shifted optimal face is a
singleton; hence $A_\Delta\setminus B_\Delta$ is contained in that null
event. Since a translation of an absolutely continuous law on $L$ is
absolutely continuous,
\eqref{eq:ac-disagreement-crossing-equivalence} follows for the fixed
displacement $\Delta$.
\end{proof}

\begin{proof}[Proof of Theorem~\ref{thm:full-support}]
Fix $\Delta\in L\setminus\{0\}$ and put $q_0=-\Delta$. The directional
width in \eqref{eq:directional-width} is positive because
$\Delta\notin L^\perp$. Define
\begin{equation}
  h(q,w):=(q+2\Delta)^\top w-\zstar(q+2\Delta),\qquad q\in L.
\end{equation}
At $q_0$, $h(q_0,w)$ equals the positive width for every
$w\in\Wstar(q_0)$. If no relative neighborhood with the asserted uniform
positivity existed, there would be $q_n\to q_0$ in $L$ and
$w_n\in\Wstar(q_n)$ with $h(q_n,w_n)\leq0$. Compactness yields a convergent
subsequence. Berge upper hemicontinuity places its limit in $\Wstar(q_0)$,
while continuity of $\zstar$
makes the limiting value of $h$ positive, a contradiction. Hence a nonempty
relatively open $U\subset L$ satisfies
\begin{equation}
  \forall q\in U,\ \forall w\in\Wstar(q):\quad
  (q+2\Delta)^\top w>\zstar(q+2\Delta).
  \label{eq:uniform-open-crossing-neighborhood}
\end{equation}
Every legal selection therefore crosses on $\{P_Lc\in U\}$. Intrinsic full
support implies $\Pr\{P_Lc\in U\}>0$, and the argument applies to every
nonzero $\Delta\in L$.
\end{proof}

\section{Counterexample Details}
\label{app:counterexamples}

\paragraph{Exact compact-support risk for Example~\ref{ex:compact-support}.}
Integrating \eqref{eq:compact-support-loss} gives
\begin{equation}
  \RiskSPOP(\hat c)=
  \begin{cases}
    \frac32-2\hat c, & \hat c\leq\frac12,\\
    2(1-\hat c)^2, & \frac12<\hat c<1,\\
    0, & \hat c\geq1.
  \end{cases}
  \label{eq:c1-full-risk}
\end{equation}
Thus the loss is strictly decreasing before the plateau and vanishes on the
entire ray $[1,\infty)$.  This calculation verifies that the displayed Bayes
set is exact rather than a partial list of additional minimizers.

The same plateau holds for a $C^\infty$ density $p$ that is symmetric about
$3/2$, strictly positive on $(1,2)$, and zero outside $[1,2]$: for
$\hat c\geq1$ the event $\{c>2\hat c\}$ has probability zero, so the
formula in \eqref{eq:compact-support-loss} gives zero loss almost surely;
for $\hat c<1$ the event $\{c>2\hat c\}$ has positive probability because
$p$ is strictly positive on $\bigl(\max\{1,2\hat c\},2\bigr)$, which is
nonempty.  Thus the phenomenon is not an artifact of
the uniform law.

\paragraph{Atomic construction for Example~\ref{ex:atomic-disagreement}.}
In this and the following singular construction, let $\Sset=[-1,1]^2$.
Let $P_c$ put mass $1/2$ at each of $(1,0)$ and $(-1,0)$, so the law is
centrally symmetric about zero.  Define the global Borel measurable selection
\begin{equation}
  \wstar(q_1,q_2)=
  \begin{cases}
    (-\operatorname{sgn}(q_1),-\operatorname{sgn}(q_2)), & q_2\neq0,\\
    (-\operatorname{sgn}(q_1),1), & q_2=0,\ q_1\leq1,\\
    (-1,-1), & q_2=0,\ q_1>1,
  \end{cases}
  \label{eq:atomic-global-oracle}
\end{equation}
where $\operatorname{sgn}(0)=0$ and $\operatorname{sgn}(x)=x/|x|$ for
$x\neq0$.  This is Borel measurable and $\wstar(q)\in\Wstar(q)$ for every
$q$: the first coordinate is the unique minimizer whenever $q_1\neq0$, and
when $q_1=0$ the value $0$ is also optimal; the second coordinate is chosen
from the optimal interval $[-1,1]$ on the tie slice $q_2=0$, and is the unique
minimizer when $q_2\neq0$.

Fix $\Delta=(\delta,0)$ with $0<\delta<1/2$.  At $c=(1,0)$, the shift to
$(1+2\delta,0)$ changes the selected second coordinate from $1$ to $-1$, but
both selections lie in the shifted face $\{-1\}\times[-1,1]$.  At
$c=(-1,0)$, both $c$ and $c+2\Delta$ have negative first coordinate and the
selected output is unchanged.  Hence
\begin{equation}
  \Pr(A_\Delta)=\frac12,
  \qquad \Pr(B_\Delta)=0,
  \qquad \crossgap{\Delta}(c)=0\quad\text{a.s.}
  \label{eq:atomic-complete-separation}
\end{equation}
The disagreement is entirely tangential to a common optimal face.

\paragraph{Singular construction for Example~\ref{ex:cantor-separation}.}
Let $X$ have the standard Cantor distribution. It is atom-free, singular, and
satisfies $X\overset d=1-X$. Define the global Borel measurable selection
\begin{equation}
  \wstar(q_1,q_2)=
  \begin{cases}
    (-\operatorname{sgn}(q_1),-\operatorname{sgn}(q_2)), & q_2\neq0,\\
    (-\operatorname{sgn}(q_1),1), & q_2=0,\ q_1\leq2,\\
    (-1,-1), & q_2=0,\ q_1>2.
  \end{cases}
  \label{eq:cantor-global-oracle}
\end{equation}
The same legality check as in the atomic example applies.  The cost
$c=(1+X,0)$ is centrally symmetric about $(3/2,0)$ and is supported on the
tie slice.  For $\Delta=(1/4,0)$, the output changes exactly on
$\{X>1/2\}$, which has probability $1/2$ by symmetry and atom-freeness.  The
old point $(-1,1)$ remains in the shifted face $\{-1\}\times[-1,1]$
throughout. Hence
\begin{equation}
  \Pr(A_\Delta)=\frac12,
  \qquad \Pr(B_\Delta)=0,
  \qquad \crossgap{\Delta}(c)=0\quad\text{a.s.}
  \label{eq:cantor-complete-separation}
\end{equation}
This closes the gap between atom-freeness and absolute continuity: the former
does not prevent a distribution from living on a Lebesgue-null tie locus.

\paragraph{The SPO+ loss is not uniformly bounded under first moments.}
Even when $\Sset$ is compact, integrability does not give a finite uniform
loss envelope.  Let $\Sset=[-1/2,1/2]$, $\hat c=0$, and $c=M>0$.  Then
\begin{equation}
  \lossSPOP(0,M)=\supportfun(M)-\zstar(M)=M\longrightarrow\infty.
  \label{eq:unbounded-spo-plus-loss}
\end{equation}
Consequently, any finite-sample argument that invokes
$\sup_{\hat c,c}\lossSPOP(\hat c,c)<\infty$ requires an additional bounded
support or tail assumption.  The population identities in this paper instead
use the stated first-moment and report-wise integrability conditions.

\paragraph{No power law from SC alone.}
In the one-dimensional problem $\Sset=[-1/2,1/2]$, take an even density that
is proportional near zero to $\exp(-1/c^2)$. Every nonzero displacement
crosses a positive-mass interval, so SC holds. Nevertheless the excess at a
positive displacement $t$ is bounded by a constant times
$t^2\exp(-1/(4t^2))$, which is smaller than every polynomial in $t$.

\section{Quantitative Proofs and Sharpness}
\label{app:quantitative-proofs}

\begin{proof}[Proof of Lemma~\ref{lem:edge-directions-span}]
If the edge directions did not span $L$, there would be a nonzero
$u\in L$ orthogonal to every edge direction. Along every edge, the linear
function $u^\top v$ is therefore constant. The vertex-edge graph of a polytope
is connected, so $u^\top v$ is constant on all vertices. Every point of the
polytope is a convex combination of vertices, hence $u^\top v$ is constant on
$\Sset$; because $u\in L$, this forces $u\in L\cap L^\perp=\{0\}$, a
contradiction. Since the unit sphere of $L$ is compact and the maximum of
finitely many continuous functions is continuous and strictly positive,
\eqref{eq:direction-covering-constant} follows.
\end{proof}

\begin{proof}[Proof of Proposition~\ref{prop:coercivity-qsc}]
Fix $\Delta$ with $0<\|\Delta\|\leq r$ and write $t=\|\Delta\|$. Let
\begin{equation}
  E_\Delta:=\left\{\crossgap{\Delta}(c)\geq\tfrac C2 t^p\right\}.
  \label{eq:coercivity-event}
\end{equation}
By Lemma~\ref{lem:centered-excess-identity} and
\eqref{eq:universal-gap-envelope},
\begin{align}
  C t^p
  &\leq\E[\crossgap{\Delta}(c)] \\
  &=\E[\crossgap{\Delta}1_{E_\Delta}]
    +\E[\crossgap{\Delta}1_{E_\Delta^c}] \\
  &\leq 2D_{\Sset}t\,\Pr(E_\Delta)+\tfrac C2t^p.
\end{align}
Rearranging gives
$\Pr(E_\Delta)\geq\tfrac{C}{4D_{\Sset}}t^{p-1}$, which is
\eqref{eq:quantitative-strict-crossing} with
$\alpha=p$, $\beta=p-1$, $\gamma=\tfrac C2$, and
$\kappa=\tfrac{C}{4D_{\Sset}}$.
\end{proof}

\begin{proof}[Proof of \eqref{eq:universal-gap-envelope}]
Let $w_0=\wstar(c)$ and $w_1=\wstar(c+2\Delta)$. Optimality of $w_0$ at $c$
gives $c^\top(w_0-w_1)\leq0$. Therefore
\begin{equation}
  \crossgap{\Delta}(c)
  =(c+2\Delta)^\top(w_0-w_1)
  \leq2\|\Delta\|\,\|w_0-w_1\|
  \leq2D_{\Sset}\|\Delta\|.
\end{equation}
\end{proof}

\begin{proof}[Proof of Theorem~\ref{thm:quantitative-crossing}]
Restrict the expectation in \eqref{eq:spo-plus-excess-identity} to the event
in \eqref{eq:quantitative-strict-crossing}. The gap is at least
$\gamma\|\Delta\|^\alpha$ on an event of probability at least
$\kappa\|\Delta\|^\beta$, proving \eqref{eq:qsc-risk-transfer}.
\end{proof}

\begin{proof}[Proof of Theorem~\ref{thm:polyhedral-quadratic-growth}]
Write $\Delta=tu$ with $u\in L$ and $\|u\|=1$. Let $\mathcal E_0$ be the finite family in
Assumption~\ref{assump:polyhedral-boundary-mass}.  By
Lemma~\ref{lem:edge-directions-span},
\begin{equation}
  \tau:=\tau_{\mathcal E_0}
  =\min_{u\in L:\|u\|=1}\max_{e\in\mathcal E_0}|a_e^\top u|>0.
  \label{eq:edge-transversality}
\end{equation}
For the direction $u$, choose an edge in $\mathcal E_0$ and orient it as
$e=[v_i,v_j]$ (relabeling its endpoints if necessary) so that
$a_e^\top u\geq\tau$, and set $\delta:=2t a_e^\top u$.  The patch property
in Assumption~\ref{assump:polyhedral-boundary-mass} is symmetric under this
relabeling.  Let $\rho_{\min}:=\min_{e\in\mathcal E_0}\rho_e$
and choose
\begin{equation}
  0<t\leq r_0:=\frac12\min\left\{\rho_{\min},
  \frac{\rho_{\min}}{\max_{e\in\mathcal E_0}\|a_e\|}\right\}.
  \label{eq:polyhedral-common-radius}
\end{equation}
Then $2t\|u-(a_e^\top u)a_e/\|a_e\|^2\|\leq\rho_e$ and
$0<\delta\leq\rho_e$. The prism
\begin{equation}
  \mathcal{P}_{e,t,u}:=
  \left\{y+\frac{s}{\|a_e\|^2}a_e:
  y\in K_e,\ -\frac{3\delta}{4}\leq s\leq-\frac{\delta}{2}\right\}
  \label{eq:crossing-prism}
\end{equation}
lies in the first vertex cone while its shift by $2tu$ lies in the second.
Indeed, write $u=u_T+(a_e^\top u)a_e/\|a_e\|^2$ with $u_T\in L_e$. The
tangential component $2tu_T$ stays within the neighborhood in
Assumption~\ref{assump:polyhedral-boundary-mass}, whereas the normal coordinate
changes from $s$ to $s+\delta$. On the prism, $s\in[-3\delta/4,-\delta/2]$,
so $s+\delta\in[\delta/4,\delta/2]$. Hence the optimizer is unique on the
prism. Since the optimizer and gap depend only on $P_Lc$ for
$\Delta\in L$, on $\{P_Lc\in\mathcal P_{e,t,u}\}$,
\begin{equation}
  \crossgap{tu}(c)=s+\delta\geq\frac{\delta}{4}
  \geq\frac{t\tau}{2}.
  \label{eq:prism-gap}
\end{equation}
The map
$T_e(y,s):=y+s a_e/\|a_e\|^2$ is an injective product coordinate map on
$K_e\times[-\rho_e,\rho_e]$. Its Jacobian with respect to $\Leb_L$ is
$\|a_e\|^{-1}$. Hence, by Fubini's theorem and the almost-everywhere
density lower bound in Assumption~\ref{assump:polyhedral-boundary-mass}, if
$A_{\min}:=\min_{e\in\mathcal E_0}\mathcal H^{k-1}(K_e)$ and
$a_{\max}:=\max_{e\in\mathcal E_0}\|a_e\|$,
\begin{equation}
  \begin{aligned}
    \Pr\{P_Lc\in\mathcal{P}_{e,t,u}\}
    &=\int_{K_e}\int_{-3\delta/4}^{-\delta/2}
      p\bigl(T_e(y,s)\bigr)\,
      \frac{ds\,d\mathcal H^{k-1}(y)}{\|a_e\|} \\
    &\geq \frac{m\,\mathcal H^{k-1}(K_e)}{\|a_e\|}\,\frac{\delta}{4}
    \geq\frac{mA_{\min}\tau}{2a_{\max}}t.
  \end{aligned}
  \label{eq:prism-mass}
\end{equation}
Multiplying \eqref{eq:prism-gap} and \eqref{eq:prism-mass} in the centered
excess identity proves \eqref{eq:polyhedral-quadratic-growth}, for example
with $C=mA_{\min}\tau^2/(4a_{\max})$ and $r=r_0$.
\end{proof}

\begin{figure*}[t]
  \centering
  \begin{tikzpicture}[x=0.82cm,y=0.82cm,>=Latex,font=\footnotesize]
  \begin{scope}
    \node[font=\bfseries\small,anchor=south west,text=PaperInk] at (-0.2,3.35) {\normalfont\footnotesize (a) a facet strip supplies mass and gap};
    \fill[PaperBlue!6] (0.25,0) rectangle (6.05,2.2);
    \draw[PaperGray,thick] (0.25,1.1) -- (6.05,1.1) node[right,font=\footnotesize] {$F_e$};
    \fill[PaperOrange!24] (1.15,0.42) rectangle (5.15,1.1);
    \draw[PaperOrange,thick,rounded corners=1pt] (1.15,0.42) rectangle (5.15,1.1);
    \node[PaperOrange,align=center,font=\footnotesize] at (3.15,1.60) {crossing strip};
    \draw[PaperOrange,thin] (3.15,1.42) -- (3.15,1.14);
    \draw[PaperOrange,thick] (1.15,0.30) -- (5.15,0.30);
    \draw[PaperOrange,thick] (1.15,0.34) -- (1.15,0.26);
    \draw[PaperOrange,thick] (5.15,0.34) -- (5.15,0.26);
    \node[PaperOrange,below,font=\footnotesize,inner sep=2pt] at (3.15,0.30)
      {width $\asymp\|\Delta\|$};
    \draw[->,PaperOrange,very thick] (5.55,0.38) -- (5.55,1.02)
      node[midway,right,font=\footnotesize] {$\Delta$};
    \node[align=center,font=\footnotesize,text=PaperGray] at (3.15,-0.72)
      {$P_c(\mathrm{strip})\gtrsim\|\Delta\|$\\[2pt]
       $G_\Delta(c)\gtrsim\|\Delta\|$};
  \end{scope}

  \begin{scope}[xshift=8.0cm]
    \node[font=\bfseries\small,anchor=south west,text=PaperInk] at (0.0,3.35) {\normalfont\footnotesize (b) one-dimensional sharp exponents};
    \draw[->] (0,0) -- (4.9,0) node[right,font=\footnotesize] {$\log|\Delta|$};
    \draw[->] (0,0) -- (0,2.25) node[above,font=\footnotesize] {$\log\,[R(\Delta)-R(0)]$};
    \draw[PaperBlue,very thick] (0.55,0.30) -- (4.35,1.55);
    \draw[PaperOrange,very thick,dashed] (0.55,0.18) -- (4.35,2.02);
    \node[PaperBlue,font=\footnotesize,anchor=east] at (4.35,0.50) {$\nu=0$: slope $2$};
    \node[PaperOrange,font=\footnotesize,anchor=east] at (4.28,2.18) {$\nu>0$: slope $2{+}\nu$};
    \node[align=center,font=\footnotesize,text=PaperGray] at (2.45,-0.72)
      {$p(c)\asymp |c|^\nu\ \Rightarrow\ R(\Delta)-R(0)\asymp|\Delta|^{2+\nu}$};
  \end{scope}
\end{tikzpicture}
  \caption{Quantitative crossing is a mass-gap product. A displacement
  sweeps a strip of width proportional to $\|\Delta\|$ across a fan facet.
  Under the stated density-floor and trimmed-interior conditions, the strip has
  probability proportional to $\|\Delta\|$ and its crossing gap is also
  proportional to $\|\Delta\|$, yielding quadratic parameter-excess growth.
  In one dimension a vanishing boundary density changes the exponent to
  $2+\nu$.}
  \label{fig:quantitative-crossing}
\end{figure*}

\begin{corollary}[Local inverse risk and conditional approximate-ERM recovery]
\label{cor:local-parameter-recovery}
Suppose that, for some $C,r,p>0$,
\begin{equation}
  \RiskSPOP(\cbar+\Delta)-\RiskSPOP(\cbar)
  \geq C\|\Delta\|^p
  \quad\text{for every }\ \Delta\in L
  \text{ with }0<\|\Delta\|\leq r.
  \label{eq:generic-local-growth}
\end{equation}
Write
$\overline B_L(\cbar,r):=\{\cbar+\Delta:\Delta\in L,
\ \|\Delta\|\leq r\}$ for the canonical effective slice.  Then every
$\hat c\in\overline B_L(\cbar,r)$ obeys
\begin{equation}
  \|\hat c-\cbar\|
  \leq C^{-1/p}\bigl(\RiskSPOP(\hat c)-\RiskSPOP(\cbar)\bigr)^{1/p}.
  \label{eq:local-inverse-risk}
\end{equation}
Let $\widehat R_n$ be an empirical criterion and let a (possibly
random) $\hat c_n\in\overline B_L(\cbar,r)$ satisfy
\begin{equation}
  \widehat R_n(\hat c_n)\leq\widehat R_n(\cbar)+\eta_n.
  \label{eq:approximate-local-erm}
\end{equation}
On any event on which
\begin{equation}
  \sup_{\theta\in\overline B_L(\cbar,r)}
  |\widehat R_n(\theta)-\RiskSPOP(\theta)|\leq\varepsilon_n,
  \label{eq:local-uniform-deviation}
\end{equation}
one has
\begin{equation}
  \|\hat c_n-\cbar\|
  \leq\left(\frac{2\varepsilon_n+\eta_n}{C}\right)^{1/p}.
  \label{eq:local-erm-parameter-rate}
\end{equation}
This statement is deterministic conditional on the local uniform-deviation
event \eqref{eq:local-uniform-deviation} and on local membership
$\hat c_n\in\overline B_L(\cbar,r)$; it is not a standalone finite-sample
rate theorem.
\end{corollary}

\begin{proof}[Proof of Corollary~\ref{cor:local-parameter-recovery}]
If $\hat c=\cbar$, \eqref{eq:local-inverse-risk} is immediate. Otherwise,
apply \eqref{eq:generic-local-growth} with $\Delta=\hat c-\cbar$ and take
$p$th roots. On the event \eqref{eq:local-uniform-deviation},
\begin{align}
  \RiskSPOP(\hat c_n)-\RiskSPOP(\cbar)
  &\leq \widehat R_n(\hat c_n)-\widehat R_n(\cbar)+2\varepsilon_n \\
  &\leq \eta_n+2\varepsilon_n.
\end{align}
Apply \eqref{eq:local-inverse-risk} to obtain
\eqref{eq:local-erm-parameter-rate}.
\end{proof}

\begin{theorem}[One-dimensional sharp exponent]
\label{thm:one-dimensional-sharp-rate}
Let $B>0$, let $\Sset=[-B,B]$, let $\cbar=0$, and let $c$ have an even
density $p$ with $\E|c|<\infty$. For $t>0$ sufficiently small,
\begin{equation}
  \RiskSPOP(t)-\RiskSPOP(0)
  =2B\int_0^{2t}(2t-x)p(-x)\,dx.
  \label{eq:one-dimensional-exact-gap}
\end{equation}
If $m x^\nu\leq p(x)\leq Mx^\nu$ for $0<x<\varepsilon$ and $\nu>-1$, then
for $0<t<\varepsilon/2$,
\begin{equation}
  \frac{2Bm(2t)^{\nu+2}}{(\nu+1)(\nu+2)}
  \leq\RiskSPOP(t)-\RiskSPOP(0)
  \leq\frac{2BM(2t)^{\nu+2}}{(\nu+1)(\nu+2)}.
  \label{eq:one-dimensional-sharp-bounds}
\end{equation}
\end{theorem}

\begin{proof}
Only $c\in(-2t,0)$ crosses the unique decision boundary. On that interval,
$\wstar(c)=B$, $\wstar(c+2t)=-B$, and
$\crossgap{t}(c)=2B(c+2t)$. Substitute $x=-c$ to obtain
\eqref{eq:one-dimensional-exact-gap}, then integrate the density bounds.
\end{proof}

\section{Contextual Prediction and Further Extensions}
\label{sec:extensions}

For a cost law $P$ governing a random cost $C$, we write
\begin{equation}
  R_{P,\mathrm{SPO+}}(\theta)
  :=\int \lossSPOP(\theta,c)\,P(dc).
  \label{eq:law-indexed-spo-plus-risk}
\end{equation}
This notation agrees with $\RiskSPOP$ in the fixed-law setup and makes the
law used by the asymmetric extensions explicit.

\begin{theorem}[Strict crossing alone does not identify the mean]
\label{thm:symmetry-free-obstruction}
For $\Sset=[-1,1]$ with $\wstar(c)=-\operatorname{sign}(c)$ for $c\neq0$ and
$\wstar(0)=0$, there exists a full-support law on $\R$ with mean zero such
that SC holds but $0\notin\argmin_\theta R_{P,\mathrm{SPO+}}(\theta)$.
\end{theorem}

Appendix~\ref{app:extensions} gives the proof of
Theorem~\ref{thm:selection-balanced-identification}, an asymmetric
counterexample for SC alone, and the approximate-symmetry localization result.
Because reflection preserves full support, Theorem~\ref{thm:full-support}
also gives exact quotient-report identification under selection balance and
full support.

For the contextual result, let $(\mathcal X,\mathcal A)$ be standard Borel,
$P_X$ a probability measure, and $K(x,\cdot)$ a Borel kernel for $C\mid X=x$.
Set $\mu(x)=\int cK(x,dc)$ and
$m(x)=\int\|c\|K(x,dc)$.

For the conditional law, define
\begin{equation}
  R_{x,\mathrm{SPO+}}(\theta)
  :=\int \lossSPOP(\theta,c)\,K(x,dc).
  \label{eq:conditional-spo-plus-risk}
\end{equation}

\begin{theorem}[Contextual SPO+ Bayes identification]
\label{thm:contextual-identification}
Assume $m(x)<\infty$ $P_X$-a.s. and $\int m\,dP_X<\infty$. Suppose there is
$E\in\mathcal A$ with $P_X(E)=1$ such that, for every $x\in E$ and every
Borel set $A\subseteq\R^d$ and every $\Delta\in L\setminus\{0\}$,
\begin{align}
 K(x,A)&=K(x,2\mu(x)-A),\nonumber\\
 K\left(x,\{c:\wstar(c)\notin\Wstar(c+2\Delta)\}\right)&>0.
 \label{eq:conditional-strict-crossing}
\end{align}
For any measurable predictor class $\mathcal F$ containing $\mu$ and satisfying
$\E\|f(X)\|<\infty$ for every $f\in\mathcal F$,
$\mathcal R(f)=\E[\lossSPOP(f(X),C)]$ is well-defined and
$\mu$ is its unique minimizer modulo the equivalence
$P_Lf(X)=P_L\mu(X)$ $P_X$-almost surely.
\end{theorem}

Appendix~\ref{app:extensions} gives the corresponding conditional quantitative
bound. The completed empirical scope is stated in
Section~\ref{sec:real-experiments} and Appendix~\ref{sec:experiment-protocol}.
Those diagnostics do not turn these population Bayes results into
finite-sample, neural-optimization, or shifted-task regret guarantees.

\section{Proofs for Symmetry Relaxation and Contextual Prediction}
\label{app:extensions}

For approximate symmetry, let $P$ denote the law of $C$ and write
$\mu=\E_P[C]$. Set $P^\circ(A)=P(2\mu-A)$,
$P^{\mathrm{sym}}=(P+P^\circ)/2$, $B_{\Sset}=\sup_{w\in\Sset}\|w\|$, and
$\eta=\sup_{h\ \mathrm{Borel},\,\|h\|_\infty\leq1}
|\int h\,d(P-P^{\mathrm{sym}})|$.
\begin{theorem}[Approximate-symmetry Bayes-set localization]
\label{thm:approximate-symmetry-localization}
If $p>1$ and, for every $\Delta\in L$ with $0<\|\Delta\|\leq r$,
\begin{equation}
 R_{P^{\mathrm{sym}},\mathrm{SPO+}}(\mu+\Delta)-R_{P^{\mathrm{sym}},\mathrm{SPO+}}(\mu)
 \geq C\|\Delta\|^p,
 \label{eq:symmetrized-coercivity}
\end{equation}
and $4B_{\Sset}\eta<Cr^{p-1}$, then the risk has a global minimizer and every
quotient minimizer $[\theta_P]$ satisfies
\begin{equation}
 \|P_L(\theta_P-\mu)\|
 \leq\left(\frac{4B_{\Sset}\eta}{C}\right)^{1/(p-1)}.
 \label{eq:approximate-symmetry-localization}
\end{equation}
This localizes the Bayes set but does not establish exact uniqueness.
\end{theorem}

\begin{corollary}[Contextual quantitative identification]
\label{cor:contextual-quantitative-identification}
Under Theorem~\ref{thm:contextual-identification}, suppose that for
$P_X$-a.e. $x$ and every $\Delta\in L$ with $0<\|\Delta\|\leq r$,
\begin{equation}
 R_{x,\mathrm{SPO+}}(\mu(x)+\Delta)-R_{x,\mathrm{SPO+}}(\mu(x))
 \geq C\|\Delta\|^p.
 \label{eq:conditional-quantitative-growth}
\end{equation}
Then every $f\in\mathcal F$ with
$\|P_L(f(X)-\mu(X))\|\leq r$ almost surely obeys
\begin{equation}
 \mathcal R(f)-\mathcal R(\mu)
 \geq C\,\E\|P_L(f(X)-\mu(X))\|^p.
 \label{eq:contextual-quantitative-identification}
\end{equation}
\end{corollary}

\begin{proof}[Proof of Theorem~\ref{thm:symmetry-free-obstruction}]
Let $Y$ have asymmetric Laplace density
\[
 f_Y(y)=\frac{ab}{a+b}
 \begin{cases}
 e^{-ay},&y\geq0,\\
 e^{by},&y<0,
 \end{cases}
 \qquad a=1,\quad b=2,
\]
and set $C=Y-\E[Y]=Y-1/2$.  The law has full support and mean zero, so the
geometric full-support argument gives SC.  At $\theta=0$, the SPO+ risk is
differentiable almost surely and
\[
 R'_{P,\mathrm{SPO+}}(0)
 =2\E[\wstar(C)-\wstar(-C)]
 =-4\E[\operatorname{sign}(C)].
\]
But $\Pr(C>0)=\Pr(Y>1/2)=\frac23e^{-1/2}\neq\frac12$, so the derivative is
nonzero.  Hence zero is not a minimizer.
\end{proof}

\begin{proof}[Proof of Theorem~\ref{thm:selection-balanced-identification}]
Set $C^\circ=2\mu-C$. Subtracting \eqref{eq:spo-plus-loss} at $\mu$ from
its value at $\mu+\Delta$, using
$\supportfun(-q)=-\zstar(q)$ and
$(C^\circ)^\top\wstar(C^\circ)=\zstar(C^\circ)$, gives the pointwise identity
\[
 \lossSPOP(\mu+\Delta,C)-\lossSPOP(\mu,C)
 =\crossgap{\Delta}(C^\circ)
 +2\Delta^\top\bigl(\wstar(C)-\wstar(C^\circ)\bigr).
\]
Compactness of $\Sset$ and integrability of $C$ justify taking expectations.
The balance assumption cancels the final term, so
\begin{equation}
 R_{P,\mathrm{SPO+}}(\mu+\Delta)-R_{P,\mathrm{SPO+}}(\mu)
 =\E[\crossgap{\Delta}(C^\circ)].
 \label{eq:selection-balanced-excess-identity}
\end{equation}
The random variable on the right is nonnegative and vanishes exactly when
$\wstar(C^\circ)\in\Wstar(C^\circ+2\Delta)$. Thus the reflected crossing
condition in Theorem~\ref{thm:selection-balanced-identification} makes the
excess strictly positive for every nonzero $\Delta\in L$. Quotient invariance then gives
$\argmin R_{P,\mathrm{SPO+}}=\mu+N$. Conversely, failure of
that condition makes the right-hand side of
\eqref{eq:selection-balanced-excess-identity} zero for a nonzero effective
displacement, so uniqueness in the quotient fails.
\end{proof}

\begin{proof}[Proof of Theorem~\ref{thm:approximate-symmetry-localization}]
For an effective displacement $\Delta\in L$, write
$D_P(\Delta)=R_{P,\mathrm{SPO+}}(\mu+\Delta)-R_{P,\mathrm{SPO+}}(\mu)$.
The report-dependent loss difference is bounded uniformly in $c$ by
$4B_{\Sset}\|\Delta\|$: the support-function difference is at most
$2B_{\Sset}\|\Delta\|$ by Lipschitz continuity, and the linear term
$2\Delta^\top\wstar(c)$ contributes at most another
$2B_{\Sset}\|\Delta\|$. Therefore
\[
 |D_P(\Delta)-D_{P^{\mathrm{sym}}}(\Delta)|
 \leq4B_{\Sset}\eta\|\Delta\|.
\]
Combining this bound with \eqref{eq:symmetrized-coercivity} for $\Delta\in L$
yields
\[
 D_P(\Delta)\geq C\|\Delta\|^p
 -4B_{\Sset}\eta\|\Delta\|
 \quad(0<\|\Delta\|\leq r).
\]
The small-asymmetry condition makes this expression positive on the boundary
of the radius-$r$ ball in $L$. Continuity and convexity give a minimizer in its
interior in the quotient; convexity rules out a lower value outside the ball.
Every report decomposes as $\mu+\Delta+n$ with $\Delta\in L$ and $n\in N$, and
the risk is invariant in $n$. Thus a quotient minimizer has representative
$\theta_P=\mu+\Delta_P$, and the displayed lower bound with
$D_P(\Delta_P)\leq0$ implies \eqref{eq:approximate-symmetry-localization}.
\end{proof}

\begin{proof}[Proof of Theorem~\ref{thm:contextual-identification}]
The loss is Borel measurable in $(x,c,\theta)$ and continuous in $\theta$.
Kernel integration therefore makes
$(x,\theta)\mapsto R_{x,\mathrm{SPO+}}(\theta)=\int\lossSPOP(\theta,c)K(x,dc)$
jointly Borel measurable.  For every $x\in E$, Theorem~\ref{thm:identification-sc} applies to the
conditional law and shows
\[
 R_{x,\mathrm{SPO+}}(\theta)=R_{x,\mathrm{SPO+}}(\mu(x))
 \quad\Longleftrightarrow\quad P_L\theta=P_L\mu(x).
\]
Integrability and Tonelli's theorem give
\[
 \mathcal R(f)-\mathcal R(\mu)
 =\E\bigl[R_{X,\mathrm{SPO+}}(f(X))
 -R_{X,\mathrm{SPO+}}(\mu(X))\bigr]\geq0.
\]
If $P_Lf(X)\neq P_L\mu(X)$ on a set of positive $P_X$ measure, the integrand
is strictly positive there, proving uniqueness in the quotient modulo
$P_X$-almost-sure equality of projected reports.
\end{proof}

\begin{proof}[Proof of Corollary~\ref{cor:contextual-quantitative-identification}]
Apply the assumed conditional lower bound at the effective displacement
$\Delta=P_L(f(X)-\mu(X))\in L$ and use quotient invariance. The radius condition
holds almost surely, so Tonelli's theorem yields the claimed inequality.
\end{proof}

\section{Real-Data Protocol and Evidence Scope}
\label{sec:experiment-protocol}

The falsifiable empirical hypothesis is that, at similar training-task decision
quality, better predictive fidelity may coincide with less degradation under a
pre-specified optimization shift. The theorems do not imply this hypothesis, and
the observed outcomes do not reveal a latent conditional Bayes report.

All three activated applications are complete. Portfolio is evaluated as a
continuous linear program (LP). Recommendation uses a complete common KuaiRec
user-item matrix, so
exact additive top-$k$ regret is available rather than logged-only evaluation.
Energy/Storage is evaluated both as an LP relaxation and as binary scheduling,
with an appended constant feature in every formal
decision-focused model. Appendix~\ref{app:real-experiments} records the full
method-by-shift matrices and ablations.

Appendix~\ref{app:semisynthetic-experiments} reports a separate mechanism run
with a known data-generating process (DGP). It keeps the three application
geometries fixed and generates
costs with a known conditional mean. Its conditional-mean, crossing, and
transfer curves are not additional tests of the original observational
hypothesis; they isolate the assumptions used by the population analysis.

For realized cost $c_i$, predicted report $\hat c_i$, shifted feasible set
$S_s$, optimizer $w_s$, and optimum value $z_s$, we report
\begin{equation}
 \widehat{R}_s=\frac1n\sum_{i=1}^n
 \left[c_i^\top w_s(\hat c_i)-z_s(c_i)\right].
 \label{eq:real-data-regret}
\end{equation}
Stochastic methods use five seeds and seed-level 95\% $t$ intervals. Complete
per-example rows are retained for bootstrap auditing, but examples are never
pooled across seeds as independent replications. Deterministic baselines are
reported as such. Root mean squared error (RMSE) is a predictive proxy, not a
measurement of quotient recovery or conditional identification.

\section{Complete Real-Data Results and Ablations}
\label{app:real-experiments}

This appendix reports every method and pre-specified shift used in the three
real-data studies.  Means are over seeds $0,\ldots,4$.  Error bars in the
figures are seed-level 95\% $t$ intervals.  Ordinary least squares (OLS) and
ridge regression (Ridge), together with the deterministic
linear KuaiRec fits produce identical rows when replayed under the five seed
indices; we mark these by $\dagger$ and do not interpret the repeated rows as
independent fits.  Per-example bootstrap intervals, checkpoints, and the exact
unrounded values remain in the accompanying CSV files.

\paragraph{Method definitions.}
All linear reports have the form $\hat c_B(x)=Bx$; Energy appends a constant
feature, while Ridge uses an unpenalized intercept.  With empirical average
$\widehat\E$, the two SPO+ objectives are
\[
  \widehat J_{\mathrm{SPO+}}(B)=\widehat\E\,\ell_{\mathrm{SPO+}}(Bx,c),
  \qquad
  \widehat J_{\mathrm{hyb}}(B)=\widehat J_{\mathrm{SPO+}}(B)+
  \lambda\widehat\E\|Bx-c\|_2^2.
\]
Here $\lambda=0.1$ in the reported hybrid cells.  Table~\ref{tab:method-definitions}
records the remaining training choices; $T$ counts parameter-update steps,
not epochs. $T$, batch size, and $\eta_0$ are listed as
Portfolio/KuaiRec/Energy, respectively.

\begin{table}[H]
\centering
\caption{Self-contained definitions of the empirical methods. The
perturbed-oracle decision-gradient method (PG) uses
Gaussian perturbations $Z$: $\sigma$ is their report-space smoothing scale and
$\eta$ is the decayed step size $\eta_0/\sqrt{t+1}$ at update step $t$.}
\label{tab:method-definitions}
\scriptsize
\begin{tabular}{p{0.17\textwidth}p{0.48\textwidth}p{0.26\textwidth}}
\hline
Method & Objective and model & Training and randomness \\
\hline
OLS / Ridge & Linear squared-report regression with ridge coefficient
$10^{-3}$; the intercept is unpenalized when used. & Deterministic. \\
Linear SPO+ & $\widehat J_{\mathrm{SPO+}}(B)$, optimized by the SPO+
subgradient through the fixed oracle. & $T=40/100/40$;
batch $=128$/full/$32$; $\eta_0=1/1/.01$. Stochastic except full-batch KuaiRec. \\
SPO+ + report & $\widehat J_{\mathrm{hyb}}(B)$; thus $\lambda$ multiplies the
squared report residual, not decision regret. & Same backbone and budget as
Linear SPO+; $\lambda=.1$. \\
PG weak report & Linear report trained with the perturbed-oracle decision
gradient $\widehat\E[(w^*(Bx+\sigma Z)-w^*(c))/\sigma]x^\top$. ``Weak
report'' means that no squared report-fidelity term is included. & Stochastic:
$n_{\rm pert}=3/16/3$ Gaussian perturbations/example; $\sigma=.01/.01/10$,
$\eta_0=.001/.01/10^{-5}$, and the same $T$; batches $128/512/32$.
Portfolio uses a low-learning-rate setting chosen for numerical stability
after validation screening; KuaiRec uses validation regret as the primary
selection criterion. \\
Neural SPO+ & KuaiRec-only MLP report $x\mapsto\hat c(x)$ trained with the
SPO+ chain-rule subgradient. & Stochastic Adam; hidden widths $(128,64)$,
100 epochs, batch $128$, learning rate $10^{-3}$, weight decay $10^{-4}$. \\
\hline
\end{tabular}
\end{table}

\begin{table}[H]
\centering
\caption{Final PG weak-report configurations. Values are five-seed means $\pm$
sample SD; the reported settings were selected using validation screening
and refit before the reported test evaluation.}
\label{tab:pg-stability-audit}
\scriptsize
\begin{tabular}{llrrr}
\hline
Dataset & PG configuration $(\sigma,\eta_0,n_{\rm pert})$ & Validation RMSE
& Test RMSE & Test shifted regret \\
\hline
Portfolio & $(.01,.001,3)$ & $0.945\pm0.029$ & $0.878\pm0.027$ & $0.01517\pm0.00009$ \\
KuaiRec & $(.01,.01,16)$ & $145.996\pm1.274$ & $150.60\pm1.50$ & $93.23\pm1.11$ \\
\hline
\end{tabular}
\end{table}

\subsection{Portfolio}

For PG, both screening grids fixed $\sigma=.01$, used five seeds, and crossed
$\eta_0\in\{1,.1,.01,.001\}$ with $n_{\rm pert}\in\{3,16\}$.
Portfolio screened 40-step fits on a fixed 512-example validation prefix.
Validation regret was nearly tied across cells, whereas RMSE and parameter
norm increased sharply with $\eta_0$; we chose $(.001,3)$ for numerical
stability, not because it minimized validation regret. The reported Portfolio
checkpoint was then refit for 40 steps and evaluated on the full validation
split and test split. KuaiRec used 100-step fits and selected $(.01,16)$ by
mean validation top-$k$ regret, with RMSE as a secondary criterion.
The original $(1,3)$ PG cell remains in both validation grids.

The chronological split is 70/15/15.  Features contain only lagged returns,
historical rolling volatility/momentum, and calendar variables.  Table
\ref{tab:portfolio-complete} and Figure~\ref{fig:portfolio-performance} show
that the pre-specified cap shift changes regret similarly for all methods although
their report RMSE differs by up to roughly sevenfold.

\begin{table}[H]
\centering
\caption{Complete Portfolio five-seed means $\pm$ sample SD.  $R_{0.10}$ is
training-task regret and $R_{0.12}$ is shifted regret. $\dagger$ denotes a
deterministic fit.}
\label{tab:portfolio-complete}
\scriptsize
\resizebox{\textwidth}{!}{%
\begin{tabular}{lccc}
\hline
Method & Report RMSE & $R_{0.10}$ & $R_{0.12}$ \\
\hline
OLS$^\dagger$ & $0.1324\pm0.0000$ & $0.013924\pm0.000000$ & $0.015078\pm0.000000$ \\
SPO+ & $0.2036\pm0.0146$ & $0.013977\pm0.000050$ & $0.015160\pm0.000058$ \\
SPO+ + report & $0.2114\pm0.0175$ & $0.013955\pm0.000088$ & $0.015120\pm0.000109$ \\
PG weak report & $0.878\pm0.027$ & $0.013972\pm0.000071$ & $0.015170\pm0.000086$ \\
\hline
\end{tabular}}
\end{table}

\begin{figure}[H]
  \centering
  \includegraphics[width=0.72\textwidth]{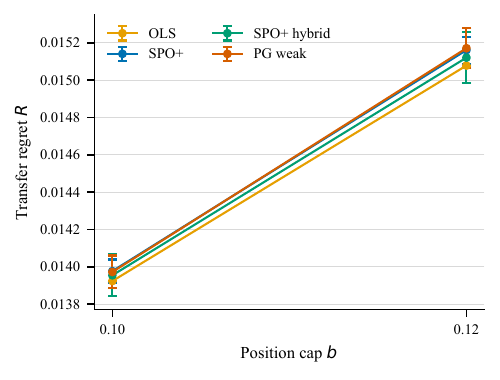}
  \caption{Portfolio regret under the training cap $b=0.10$ and shifted cap
  $b=0.12$.  Error bars are seed-level 95\% intervals; deterministic OLS has no
  seed variation.}
  \label{fig:portfolio-performance}
\end{figure}

The report-penalty ablation uses
$\lambda\in\{0,0.01,0.1,1\}$.  At $b=0.12$, $\lambda=0.01$ gives
$0.015164\pm0.000067$ regret (sample SD) and RMSE
$0.207\pm0.025$.  The $\lambda=1$ endpoint has similar regret
($0.015192\pm0.000082$) but RMSE
$(3.14\pm2.74)\times10^6$ and degenerate crossing slopes.  Thus
$\lambda\in\{0,0.01,0.1\}$ supplies a stable small-grid check. The
$\lambda=1$ result is numerically unstable under the fixed step-size schedule
and does not establish a geometric failure boundary.

\begin{figure}[H]
  \centering
  \includegraphics[width=0.99\textwidth]{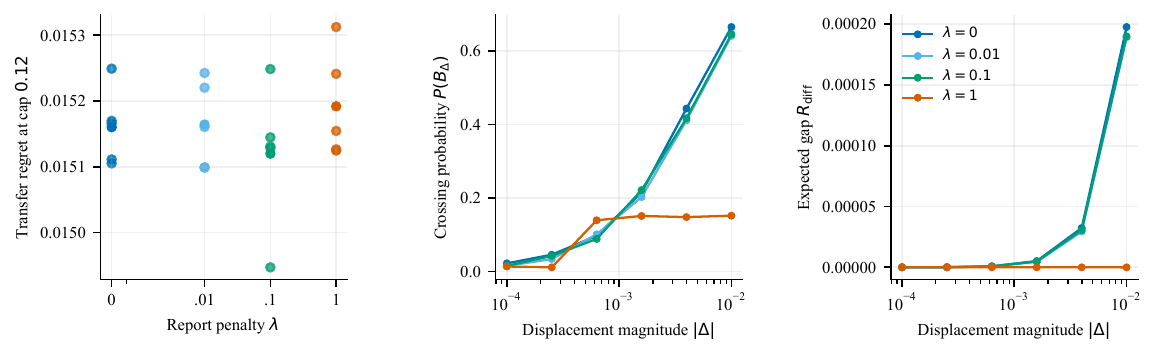}
  \caption{Portfolio report-penalty ablation.  Left: five seed outcomes at the
  shifted cap.  Middle and right: symmetrized crossing probability and
  unconditional expected transition gap versus displacement.  The
  $\lambda=1$ model has exploding RMSE; its crossing slopes are unreliable.}
  \label{fig:portfolio-penalty}
\end{figure}

\begin{figure}[H]
  \centering
  \includegraphics[width=0.96\textwidth]{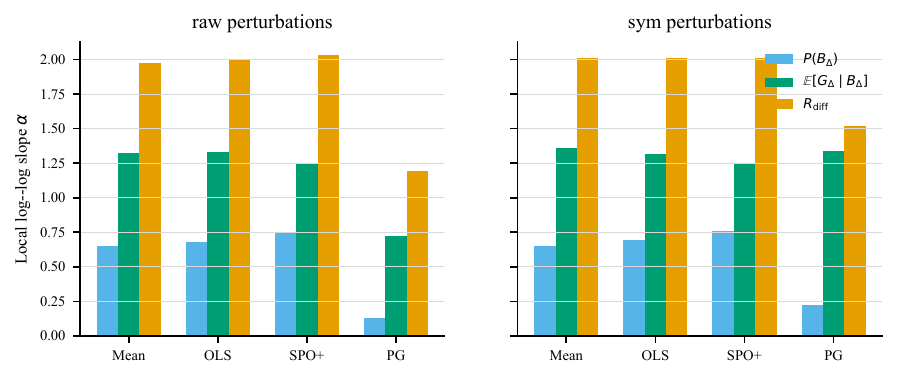}
  \caption{Portfolio local log-log slopes for raw and symmetrized
  perturbations. Mean, OLS, and SPO+ give the decomposition
  $P(B_\Delta)\E[G_\Delta\mid B_\Delta]=R_{\mathrm{diff}}$ with an
  approximately quadratic product. PG is retained as a contrast. These
  finite-sample diagnostics do not verify symmetry or strict crossing of the
  population law.}
  \label{fig:portfolio-real-geometry}
\end{figure}

\subsection{Complete-matrix KuaiRec recommendation}

Users are disjoint across train, validation, and test splits.  Candidates with
any missing feedback in the common matrix are removed rather than imputed.
Costs are therefore observed for every candidate used in exact additive top-$k$
regret.  This is not a logged-bandit or off-policy evaluation (OPE) or
position-counterfactual experiment.

\begin{table}[H]
\centering
\caption{Complete KuaiRec five-seed means $\pm$ sample SD. The MLP report model
is auxiliary; all other rows are pre-specified formal methods. $\dagger$ marks
deterministic linear fits.}
\label{tab:recommendation-complete}
\small
\resizebox{\textwidth}{!}{%
\begin{tabular}{lrrrr}
\hline
Method & Report RMSE & $R_{k=5}$ & $R_{k=10}$ & $R_{k=15}$ \\
\hline
Ridge$^\dagger$ & $58.92\pm0.00$ & $53.96\pm0.00$ & $71.00\pm0.00$ & $82.68\pm0.00$ \\
Linear SPO+$^\dagger$ & $71.73\pm0.00$ & $\mathbf{50.63}\pm0.00$ & $\mathbf{66.25}\pm0.00$ & $76.86\pm0.00$ \\
SPO+ + report$^\dagger$ & $\mathbf{58.04}\pm0.00$ & $50.99\pm0.00$ & $66.71\pm0.00$ & $77.43\pm0.00$ \\
PG weak report & $150.60\pm1.50$ & $57.76\pm0.56$ & $78.77\pm1.01$ & $93.23\pm1.11$ \\
Neural SPO+ (pre-specified) & $100.53\pm2.31$ & $53.16\pm0.20$ & $70.72\pm0.33$ & $83.23\pm0.25$ \\
MLP (MSE, auxiliary) & $58.71\pm0.02$ & $52.03\pm0.14$ & $67.88\pm0.28$ & $78.67\pm0.17$ \\
\hline
\end{tabular}}
\end{table}

\begin{figure}[H]
  \centering
  \includegraphics[width=0.66\textwidth]{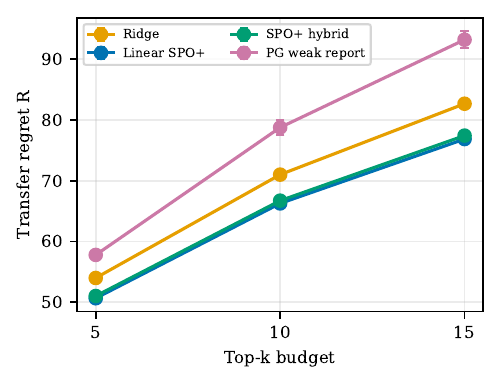}
  \caption{Complete KuaiRec regret across all pre-specified top-$k$ budgets.
  Linear SPO+ has the lowest formal-method mean at each $k$; the hybrid is
  close, and validation-selected PG remains less competitive. Error bars are seed-level 95\%
  intervals.}
  \label{fig:recommendation-budget}
\end{figure}

\begin{figure}[H]
  \centering
  \includegraphics[width=0.58\textwidth]{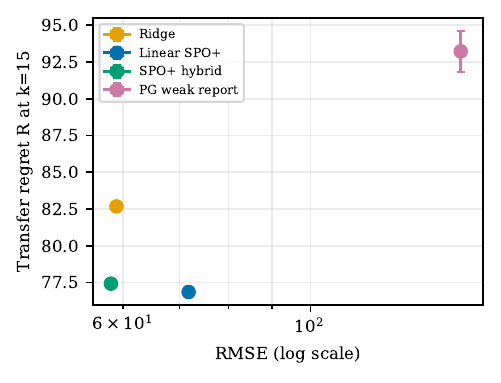}
  \caption{KuaiRec report RMSE versus $k=15$ regret.  Horizontal and vertical
  bars are seed-level 95\% intervals.  Linear SPO+ improves regret relative to
  Ridge, while the hybrid improves linear report RMSE; PG remains less competitive
  after validation-selected stabilization.}
  \label{fig:recommendation-main}
\end{figure}

The Neural SPO+ ablation varies epochs in $\{50,100\}$ and hidden widths in
$\{64,(128,64)\}$.  Within this finite grid, the 50-epoch, width-64 model has
the lowest observed regret, while the 50-epoch, $(128,64)$ model has the
lowest RMSE.  This is an optimization/capacity sensitivity check, not a test
of SC, quotient-report recovery, or neural-model-class consistency; it does
not attribute either outcome to the architecture.

\begin{table}[H]
\centering
\caption{Neural SPO+ capacity/epoch sensitivity check (five-seed means $\pm$
sample SD), not an identification or strict-crossing test.}
\label{tab:recommendation-neural-ablation}
\small
\begin{tabular}{lrrrr}
\hline
Configuration & RMSE & $R_{5}$ & $R_{10}$ & $R_{15}$ \\
\hline
50 epochs, 64 & $73.35\pm0.18$ & $\mathbf{50.67}\pm0.08$ & $\mathbf{65.96}\pm0.17$ & $\mathbf{76.55}\pm0.29$ \\
50 epochs, $128\times64$ & $\mathbf{72.32}\pm0.27$ & $51.07\pm0.04$ & $66.58\pm0.17$ & $77.36\pm0.23$ \\
100 epochs, 64 & $85.71\pm0.73$ & $51.99\pm0.12$ & $68.76\pm0.29$ & $80.36\pm0.25$ \\
100 epochs, $128\times64$ & $100.53\pm2.14$ & $53.16\pm0.20$ & $70.72\pm0.33$ & $83.23\pm0.25$ \\
\hline
\end{tabular}
\end{table}

\begin{figure}[H]
  \centering
  \includegraphics[width=0.99\textwidth]{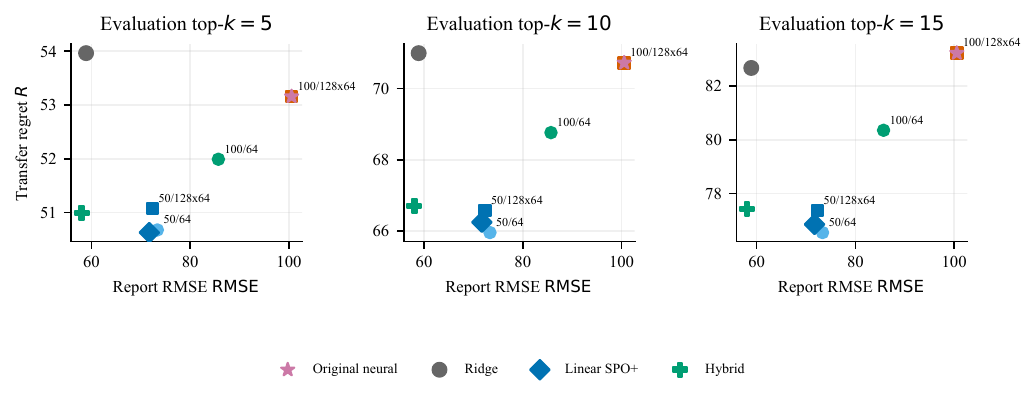}
\caption{KuaiRec RMSE versus regret on the Pareto plane at $k\in\{5,10,15\}$.
Labels give epochs/hidden widths. The originally pre-specified
100/$128\times64$ network is shown separately from the ablation-selected
50-epoch models.  This finite-grid sensitivity check does not measure
quotient-report error or establish strict crossing.}
  \label{fig:recommendation-neural-pareto}
\end{figure}

Figure~\ref{fig:recommendation-geometry} reports exact top-$k$ crossing
diagnostics. Positive crossing and transition gaps occur for Ridge,
Linear SPO+, and the hybrid at larger perturbations; the weak-report PG model
is flat in the tested neighborhood.  These curves diagnose the fitted reports
under the finite candidate matrix and do not identify a population Bayes cost.

\begin{figure}[H]
  \centering
  \includegraphics[width=0.92\textwidth]{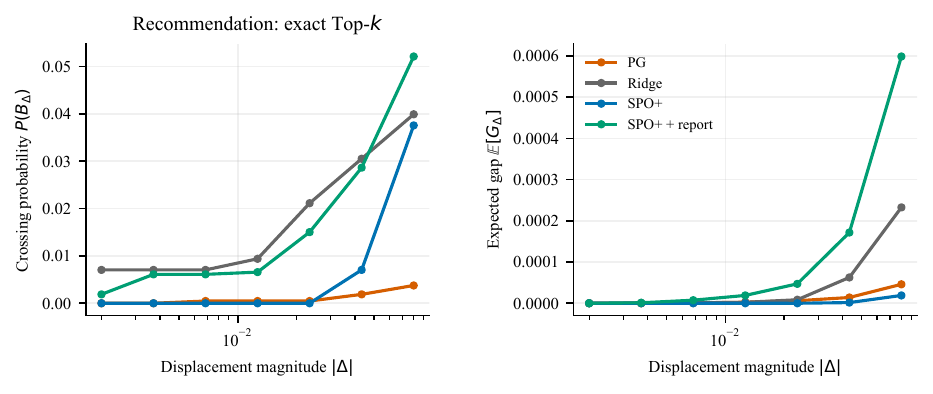}
  \caption{Complete-matrix KuaiRec top-$k$ crossing probability and expected
  transition gap.  The two panels are separated sufficiently for independent
  axis labels and share the displacement scale.}
  \label{fig:recommendation-geometry}
\end{figure}

\subsection{Energy/Storage}

Energy uses complete 48-period days, a chronological 70/15/15 split, and only
past or contemporaneously available features.  All formal decision-focused
fits standardize the decision features and append a constant feature.  The LP
and binary optimizers share the same 48-period scheduling instance.  Table
\ref{tab:energy-complete} shows near-matching mean regrets on these paired
evaluations; this does not establish integrality of the relaxation. In fact,
the LP can return fractional schedules for this instance.

\begin{table}[H]
\centering
\caption{Complete intercept-corrected Energy five-seed means $\pm$ sample SD.
Values are in thousands. Ridge is deterministic ($\dagger$).}
\label{tab:energy-complete}
\scriptsize
\resizebox{\textwidth}{!}{%
\begin{tabular}{lrrrrrrr}
\hline
& & \multicolumn{3}{c}{LP relaxation} & \multicolumn{3}{c}{Binary scheduling} \\
Method & RMSE & $s=0$ & $s=0.5$ & $s=1$ & $s=0$ & $s=0.5$ & $s=1$ \\
\hline
Ridge$^\dagger$ & $31.65\pm0.00$ & $71.10\pm0.00$ & $52.96\pm0.00$ & $51.10\pm0.00$ & $71.11\pm0.00$ & $52.97\pm0.00$ & $51.09\pm0.00$ \\
SPO+ & $111.30\pm7.56$ & $73.14\pm29.12$ & $58.89\pm26.36$ & $\mathbf{48.29}\pm12.23$ & $73.12\pm29.12$ & $58.88\pm26.36$ & $\mathbf{48.29}\pm12.22$ \\
SPO+ + report & $107.02\pm6.23$ & $72.59\pm23.63$ & $56.76\pm22.72$ & $50.16\pm10.97$ & $72.59\pm23.63$ & $56.76\pm22.72$ & $50.16\pm10.97$ \\
PG weak report & $74.67\pm0.00$ & $573.07\pm102.57$ & $700.72\pm151.36$ & $783.80\pm197.00$ & $573.07\pm102.57$ & $700.72\pm151.36$ & $783.79\pm197.00$ \\
\hline
\end{tabular}}
\end{table}

\begin{figure}[H]
  \centering
  \includegraphics[width=0.96\textwidth]{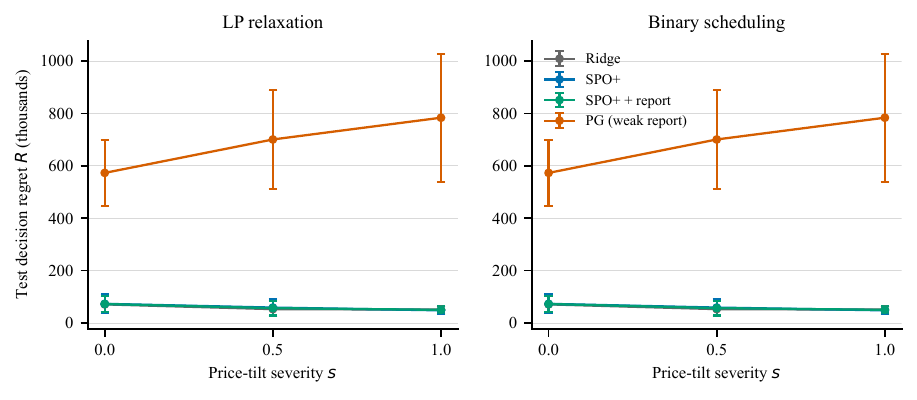}
  \caption{Intercept-corrected Energy regret for the LP relaxation and binary
  scheduler, shown in thousands.  Ridge is deterministic; other bars are
  seed-level 95\% intervals. SPO+ wins only at the largest pre-specified tilt,
  while PG remains unstable and substantially worse.}
  \label{fig:energy-main}
\end{figure}

At $s=1$, the report-penalty grid gives LP regrets $48.29$, $49.73$, $50.16$,
and $120.56$ thousand for $\lambda=0,.01,.1,1$, respectively.  The iteration
grid gives $251.86$, $48.29$, and $42.43$ thousand for $T=10,40,80$.
The best tested PG cell ($\sigma=1$, $\eta=10^{-5}$) still gives $759.87$
thousand. Binary values are nearly identical. SPO+ improves with the tested
training budget; the largest report penalty and tested PG settings are
numerically unstable. The grid does not establish a universally optimal report penalty.

\begin{figure}[H]
  \centering
  \includegraphics[width=0.99\textwidth]{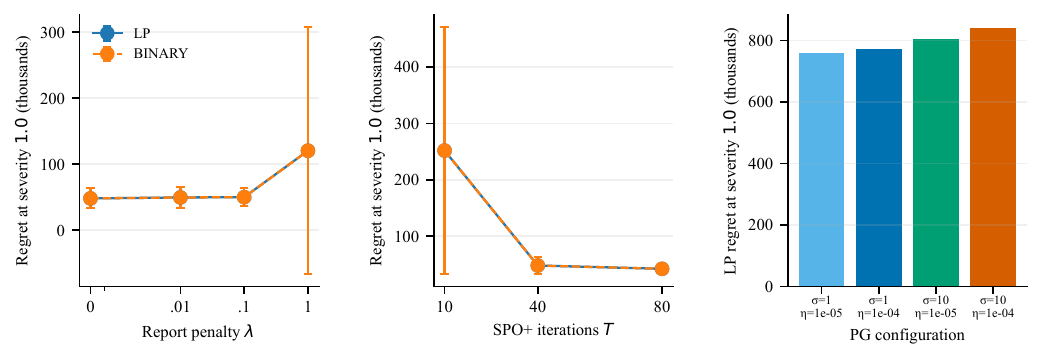}
  \caption{Intercept-corrected Energy ablations at $s=1$, with regret in
  thousands.  Left: report penalty; middle: SPO+ training budget; right: PG
  stability grid.  LP and binary curves overlap almost exactly.}
  \label{fig:energy-ablations}
\end{figure}

The method-averaged LP curves in Figure~\ref{fig:energy-geometry} increase
with relative displacement for all fitted reports.  Individual direction cells
need not be monotone; local slopes vary by method and direction.  This remains
qualitative mechanism evidence, not a claim of the uniform boundary-mass
condition or a universal quadratic rate.

\begin{figure}[H]
  \centering
  \includegraphics[width=0.94\textwidth]{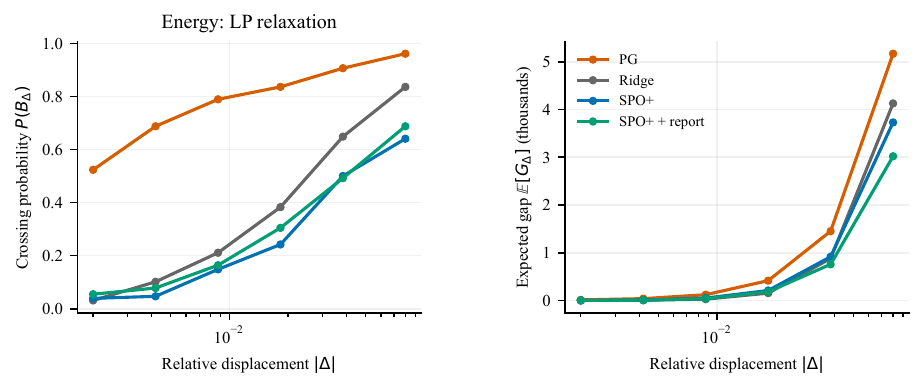}
  \caption{Energy LP crossing probability and expected transition gap (in
  thousands) versus relative displacement.  Curves are finite-sample
  fitted-report diagnostics; no population symmetry or boundary-density
  assumption is inferred.}
  \label{fig:energy-geometry}
\end{figure}

\subsection{Cross-application synopsis}

Figure~\ref{fig:cross-application-overview} normalizes each method's regret by
its own pre-specified baseline within the corresponding application. It summarizes
sensitivity to the pre-specified shifts; it does not make objective scales
comparable, supply a hypothesis test, or turn the three applications into
independent replications of one estimand.

\begin{figure}[H]
  \centering
  \includegraphics[width=0.98\textwidth]{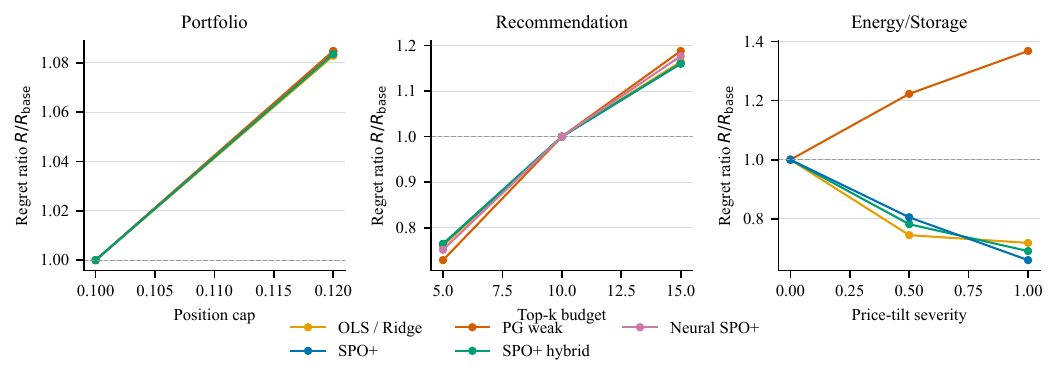}
  \caption{Cross-application shift profiles. Each curve reports regret relative
  to that method's application-specific pre-specified baseline.}
  \label{fig:cross-application-overview}
\end{figure}

\subsection{Artifact provenance}

The supplementary machine-readable materials contain the unrounded summary
tables, per-example records, checkpoint metadata, ablation outputs, and the
cross-application audit inventory. GPU Scheduling is not part of the final
empirical scope because its replay protocol was not suitable for a formal
claim.

\section{Experiments with Known Data-Generating Processes on Real Application Geometries}
\label{app:semisynthetic-experiments}

The observational experiments in Section~\ref{sec:real-experiments} cannot
reveal $\mu(x)=\E[C\mid X=x]$, conditional symmetry, or population strict
crossing. We therefore add a semi-synthetic layer that preserves the feature
rows and optimization geometry of each application while replacing the
realized costs with a registered law whose conditional mean is known:
\[
  C=\mu(X)+\varepsilon.
\]
The experiment is designed to isolate the mechanism in the theorems. It is
not a reanalysis of the original observational outcomes.

\paragraph{Data-generating processes (DGPs) and evaluation.}
For Portfolio, KuaiRec, and Energy/Storage, the adapters reuse the existing
data loaders and application-specific optimizers. The \textsc{sc-on} law uses
a feature-dependent mean and centered diagonal Gaussian noise, with each
coordinate scaled by the corresponding training-cost dispersion and a fixed
multiplier of $3$. The \textsc{sc-off} control uses a deterministic conditional
law at a locally stable reference cost. The conditional-mean diagnostic
repeats outcomes at the same feature value and measures the RMSE of their
sample mean relative to the supplied $\mu(x)$. It therefore measures Monte
Carlo recovery under the known law; it does not fit or identify a regression
function from the original data.

For each application, the crossing diagnostic evaluates the transition gap
$G_\Delta$ on a fixed effective-direction and displacement grid. We report
crossing mass, expected gap, and conditional mean recovery. The transfer
diagnostic evaluates the known mean and controlled perturbations over the
registered shift families: position caps for Portfolio, top-$k$ values for
KuaiRec, and price tilts for Energy/Storage. The shift family is finite by
construction.

The current output is a mechanism and figure run rather than the
preregistered high-budget confirmation. It used $8$ feature rows and $16$
Monte Carlo repetitions for Portfolio, $8$ and $8$ for KuaiRec, and $3$ and
$8$ for Energy/Storage. The mean-recovery grid used $1,4,16,64$ repeated
outcomes. The run used three effective directions and five displacements for
Portfolio and KuaiRec, and two directions and four displacements for
Energy/Storage.

\paragraph{Results.}
The conditional-mean curves decrease with the number of repeated outcomes in
all three applications under \textsc{sc-on}; the deterministic \textsc{sc-off}
control reaches numerical zero. The crossing curves show positive mass under
\textsc{sc-on} on part of the registered grid. Under \textsc{sc-off}, the
smallest displacements remain locally flat in Portfolio and KuaiRec, while
larger displacements leave the local region and produce crossing in some
application geometries. In the current run, positive crossing occurred in
$16/36$, $5/36$, and $10/20$ direction-scale rows for \textsc{sc-on} in
Portfolio, KuaiRec, and Energy/Storage, respectively; the corresponding
\textsc{sc-off} counts were $2/36$, $0/36$, and $5/20$.

These results support a mechanism statement: the same application geometries
can exhibit the predicted distinction between a full-support noisy law and a
locally flat control when the conditional law is known. They do not support a
claim that the original data satisfy strict crossing. The finite grid also
does not establish the all-displacement quantifier in the population
definition.

\begin{figure*}[t]
  \centering
  \includegraphics[width=0.98\textwidth]{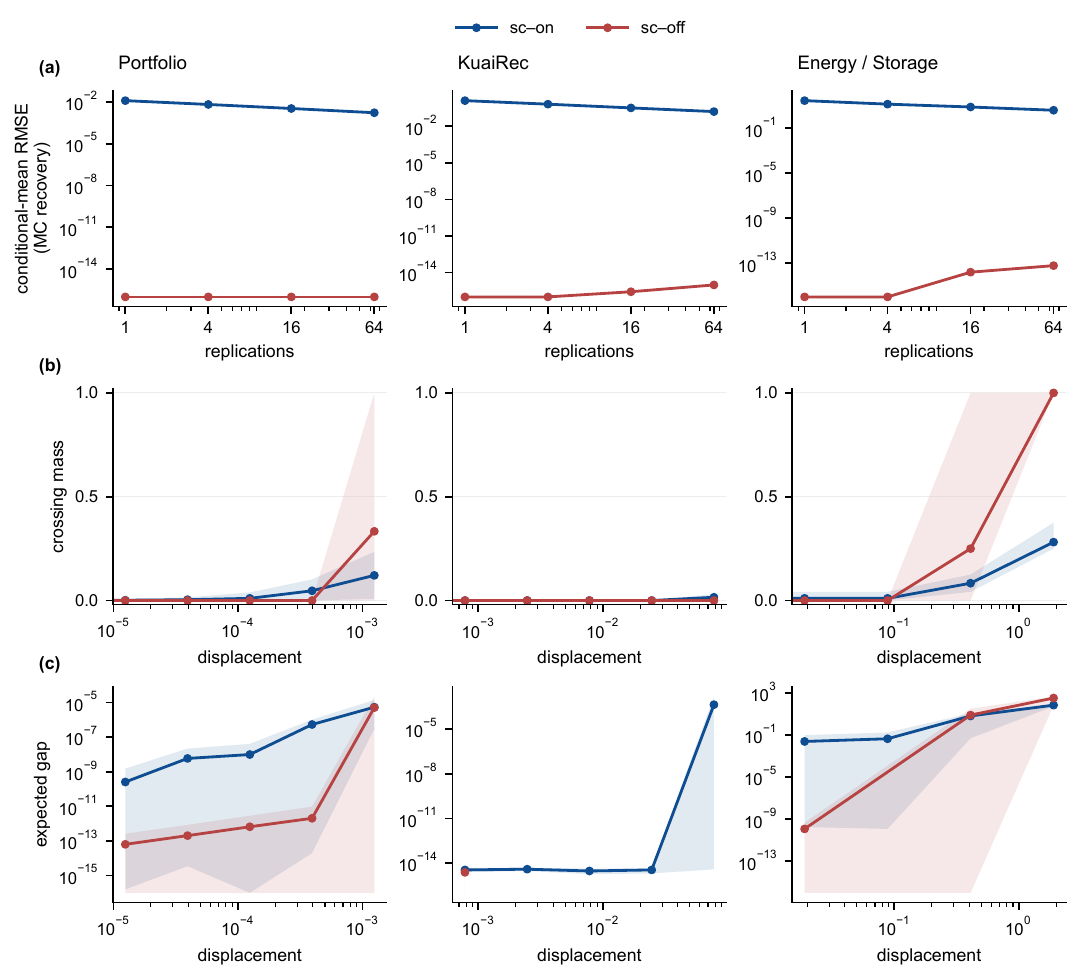}
  \caption{Known-DGP identification and crossing diagnostics on the real
  Portfolio, KuaiRec, and Energy/Storage geometries. The top row reports
  conditional-mean Monte Carlo recovery as the number of repeated outcomes
  increases. The middle row reports crossing mass over the tested effective
  directions and displacements. The bottom row reports expected transition
  gaps, with ribbons spanning the tested directions and signs. The curves are
  finite-grid diagnostics; they do not certify population strict crossing or
  conditional-mean identification from the original observational data.}
  \label{fig:semisynthetic-identification}
\end{figure*}

\begin{figure*}[t]
  \centering
  \includegraphics[width=0.98\textwidth]{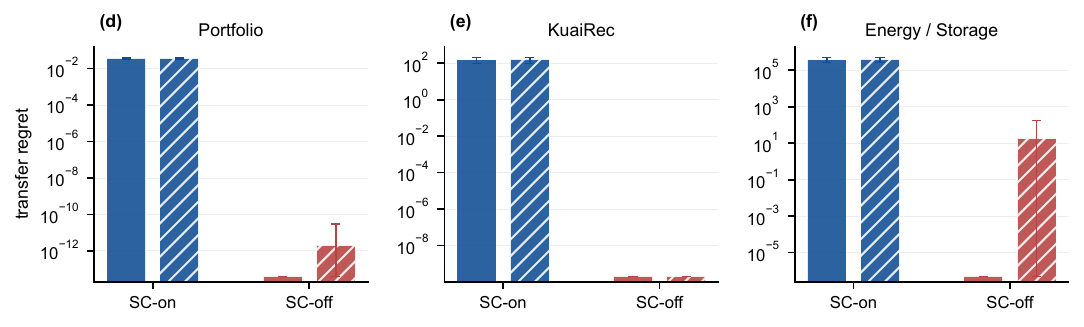}
  \caption{Transfer stress test under the registered finite shift families.
  Solid bars use the known conditional mean and hatched bars use controlled
  report perturbations. Error bars span the registered shifts and controlled
  perturbations within each application and regime. The logarithmic axis keeps
  the small \textsc{sc-off} regrets visible. The plot is finite-family evidence and
  does not establish a universal transfer guarantee.}
  \label{fig:semisynthetic-transfer}
\end{figure*}

The transfer plot should be read together with the finite-shift restriction.
It tests whether controlled report perturbations change regret under the
specified application shifts. It does not quantify transfer over arbitrary
objectives, constraints, feasible sets, or distributions. In particular, a
known conditional mean in the semi-synthetic law removes one source of
uncertainty; it does not supply the assumptions required for a universal
transfer theorem.

\section{Deterministic Illustrations}
\label{sec:illustrations}

Figure~\ref{fig:deterministic-illustrations} visualizes exact one-dimensional
formulae and the analytic probabilities in the two oracle-separation examples;
it uses neither Monte Carlo estimates nor fitted rates. In the compact-support
construction, the plotted risk is $\E[(c-2\hat c)_+]$ for
$c\sim\mathrm{Unif}[1,2]$, hence it is exactly zero for $\hat c\geq1$. For
the boundary-density family on $[-1,1]$ with
$p_\nu(c)=(\nu+1)|c|^\nu/2$, the plotted excess is the closed form in
Appendix~\ref{app:additional-illustrations}. The final panel records that C4
and C5 have positive selected-output disagreement but zero crossing
probability and zero excess gap. The panels display these mechanisms but do not
supply evidence for the theorems.

\begin{figure*}[t]
  \centering
  \includegraphics[width=\textwidth]{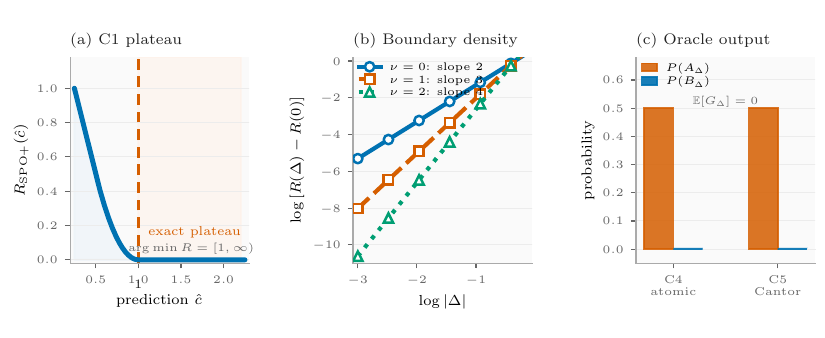}
  \caption{Closed-form illustrations: the C1 plateau, boundary-density slopes
  $2+\nu$, and oracle disagreement without crossing. These are analytic
  constructions, not empirical evidence.}
  \label{fig:deterministic-illustrations}
\end{figure*}

\section{Details for the Deterministic Illustrations}
\label{app:additional-illustrations}

All curves in Figure~\ref{fig:deterministic-illustrations} are evaluated from
closed forms.  The first panel uses the exact risk in
\Eqref{eq:c1-full-risk} for Example~\ref{ex:compact-support}.
For the centered one-dimensional model with $\Sset=[-B,B]$ and
$p_\nu(c)=(\nu+1)|c|^\nu/2$ on $[-1,1]$, direct integration gives, for
$0<|t|<1/2$,
\begin{equation}
  \RiskSPOP(t)-\RiskSPOP(0)
  =\frac{B(2|t|)^{\nu+2}}{\nu+2}.
  \label{eq:one-dimensional-closed-form-plot}
\end{equation}
The third panel plots the exact probabilities in
Examples~\ref{ex:atomic-disagreement} and \ref{ex:cantor-separation}:
$\Pr(A_\Delta)=1/2$ and $\Pr(B_\Delta)=0$. The vector figure renders these
closed forms and contains no random sampling.

\clearpage
\section{Fenchel Transition Extension (Supporting)}
\label{sec:generalization}
\label{app:generalization}

The support-function identity has a precise Fenchel-transition extension, but
this appendix is a supporting extension, not a general decision-focused
surrogate identification theorem.  The normal-fan geometry, decision
consistency, and polyhedral crossing interpretation remain tied to the
SPO+/support-function structure.  In this appendix only, costs and reports lie
in $\R^m$. Let $\Omega$ be proper, lower-semicontinuous, and convex, let
$\Phi:=\Omega^*$ be finite on $\R^m$, and fix a measurable
selection $s(q)\in\partial\Phi(q)$. Define
\begin{equation}
  \mathcal{W}_{\Omega}(c):=\partial\Phi(-c),
  \qquad
  w_{\Omega}(c):=s(-c),
  \label{eq:regularized-decision-map}
\end{equation}
and consider the Fenchel transition loss
\begin{equation}
  \ell_{\Phi}(\theta,c):=
  \Phi(c-2\theta)+2\theta^\top w_{\Omega}(c)+\Phi(-c)+a(c),
  \label{eq:fenchel-transition-loss}
\end{equation}
where $a(c)$ is integrable and does not depend on $\theta$.

Write
\begin{equation}
  R_{\Phi}(\theta):=\E[\ell_{\Phi}(\theta,c)]
  \label{eq:fenchel-transition-risk}
\end{equation}
for the corresponding population risk.

\begin{theorem}[Fenchel transition identity]
\label{thm:fenchel-transition-identity}
Suppose $c\overset d=2\cbar-c$ and, for every $\theta\in\R^m$, each of
$\Phi(c-2\theta)$, $\theta^\top s(-c)$, $\Phi(-c)$, and $a(c)$ is integrable.
Then
\begin{equation}
  R_{\Phi}(\cbar+\Delta)-R_{\Phi}(\cbar)
  =\E\left[D_{\Phi}(-c-2\Delta,-c;s(-c))\right],
  \label{eq:fenchel-transition-identity}
\end{equation}
where $D_{\Phi}(x,y;s(y)):=\Phi(x)-\Phi(y)-s(y)^\top(x-y)$.
\end{theorem}

\begin{proof}[Proof of Theorem~\ref{thm:fenchel-transition-identity}]
Subtract \eqref{eq:fenchel-transition-loss} at $\cbar$ from its value at
$\cbar+\Delta$. The offset $a(c)$ cancels and central symmetry gives
\begin{align}
  R_{\Phi}(\cbar+\Delta)-R_{\Phi}(\cbar)
  &=\E\bigl[\Phi(-c-2\Delta)-\Phi(-c)+2\Delta^\top s(-c)\bigr] \\
  &=\E\bigl[D_{\Phi}(-c-2\Delta,-c;s(-c))\bigr].
\end{align}
The last equality is the definition of the selected Bregman divergence.
\end{proof}

Because the selected Bregman divergence vanishes exactly when
$s(y)\in\partial\Phi(x)$, the target is uniquely identified precisely when
this selected divergence is positive with positive probability for every
nonzero displacement.

\begin{remark}[Scope of the extension]
For $\Omega=I_{\Sset}$, $\Phi=\supportfun$, and the support-function
selection must be aligned as $s(q):=\wstar(-q)$. Then
$w_{\Omega}(c)=\wstar(c)$ and
$D_{\Phi}(-c-2\Delta,-c;\wstar(c))=\crossgap{\Delta}(c)$, recovering
Theorem~\ref{thm:identification-sc}. Quadratic potentials instead identify by
pointwise curvature, while entropy identifies costs only modulo additive
constants. No generic SPO task-consistency statement follows from
Theorem~\ref{thm:fenchel-transition-identity}.
\end{remark}

\begin{theorem}[Fenchel identification by strict Bregman crossing]
\label{thm:fenchel-identification-bsc}
Under the assumptions of Theorem~\ref{thm:fenchel-transition-identity},
$\cbar$ is the unique minimizer of $R_{\Phi}$ if and only if, for every
$\Delta\neq0$,
\begin{equation}
  \Pr\{w_{\Omega}(c)\notin\mathcal{W}_{\Omega}(c+2\Delta)\}>0.
  \label{eq:strict-bregman-crossing}
\end{equation}
\end{theorem}

\begin{proof}
For a finite convex $\Phi$ and $s(y)\in\partial\Phi(y)$,
\begin{equation}
  D_{\Phi}(x,y;s(y))=0
  \quad\Longleftrightarrow\quad s(y)\in\partial\Phi(x).
  \label{eq:bregman-equality-subgradient}
\end{equation}
Apply \eqref{eq:bregman-equality-subgradient} to the nonnegative expectation
in \eqref{eq:fenchel-transition-identity}. The same strict-positive-expectation
argument as in Theorem~\ref{thm:identification-sc} proves the result.
\end{proof}

\begin{theorem}[Fenchel absolute-continuity equivalence]
\label{thm:fenchel-ac-equivalence}
Let $\mathcal{T}_{\Phi}:=\{q:\partial\Phi(q)\text{ is non-singleton}\}$.
If $P_c\ll\Leb_m$, then for every fixed $\Delta$,
\begin{equation}
  \Pr\left(
    \{w_{\Omega}(c)\neq w_{\Omega}(c+2\Delta)\}
    \SymDiff
    \{w_{\Omega}(c)\notin\mathcal{W}_{\Omega}(c+2\Delta)\}
  \right)=0.
\end{equation}
\end{theorem}

\begin{proof}
A finite convex function is differentiable Lebesgue-almost everywhere, so
$\Leb_m(\mathcal{T}_{\Phi})=0$. Outside the translated tie locus, the new
subdifferential is a singleton and output disagreement forces crossing. The
reverse event inclusion is always true.
\end{proof}

\section{Mechanism Experiments}
\label{app:mechanism-experiments}

The four mechanism experiments reproduce the predicted geometry on synthetic
laws with closed-form or numerically exact objectives. Their curves are
deterministic evaluations of these constructions.

\paragraph{Sharp exponent and absence of polynomial modulus (Exps.~1 and 2).}
For $\Sset=[-1,1]$, $\cbar=0$, and
$p_\nu(c)=(\nu+1)|c|^\nu/2$ on $[-1,1]$,
Theorem~\ref{thm:one-dimensional-sharp-rate} gives the exact closed form
$\RiskSPOP(t)-\RiskSPOP(0)=B(2|t|)^{\nu+2}/(\nu+2)$.
Figure~\ref{fig:mech-A}(a) plots this exact curve for
$\nu\in\{0,0.5,1,2\}$; the log-log slopes are exactly $2+\nu$.  For the
$C^\infty$ flat density $p(c)\propto e^{-1/|c|}$, strict crossing holds for
every $t>0$, yet the excess risk falls below $t^2,t^3,t^4$; the local
log-log slope grows without bound as $t\downarrow 0$
(Figure~\ref{fig:mech-A}(b)).  SC alone carries no polynomial modulus.

\paragraph{Oracle disagreement versus optimal-face crossing (Exp.~3).}
Figure~\ref{fig:mech-B} shows the atomic tie example
(Example~\ref{ex:atomic-disagreement}) and the Cantor example
(Example~\ref{ex:cantor-separation}) under increasing noise
$c_\varepsilon=c+\varepsilon Z$.  At $\varepsilon=0$,
$\Pr(A_\Delta)=1/2$ while $\Pr(B_\Delta)=0$ and $\E[G_\Delta]=0$: the
selected oracle moves while the optimal face is never departed.  For
$\varepsilon>0$ the law is absolutely continuous and
$\Pr(A_\Delta)=\Pr(B_\Delta)$
(Theorem~\ref{thm:ac-equivalence}); both grow with $\varepsilon$.  The
gap-based events are insensitive to the tolerance
$\mathrm{tol}\in\{0,10^{-8},10^{-6}\}$.

\paragraph{Normal-fan mass yields quadratic growth (Exp.~4).}
For the box $\Sset=[-1,1]^2$ with $\cbar=(2,2)$ and $\mathrm{Unif}(0,4)^2$
costs, the mechanism
$P(B_{tu})\asymp t$, $\E[G_{tu}\mid B_{tu}]\asymp t$,
$\RiskSPOP(\cbar+tu)-\RiskSPOP(\cbar)\asymp t^2$ holds uniformly over the
tested directions (Figure~\ref{fig:mech-C}(a), slopes $1/1/2$); inward
directions have no crossing and flat risk.  Removing the boundary patch near
one facet (symmetrically about $\cbar$) creates an exact dead zone
$P(B_{tu})=0$ and flat risk for $t$ below the gap, with growth resuming
later, while the intact direction is unaffected
(Figure~\ref{fig:mech-C}(b)).  Lowering the density floor keeps the slopes
and shrinks the constants.  The same mechanism holds for the non-symmetric
triangle $\mathrm{conv}\{(0,0),(2,0),(0,1)\}$
(Figure~\ref{fig:mech-C}(c)), computed by vertex enumeration without
solver-output comparisons.

\paragraph{Controlled neural quotient-report diagnostic (Exp.~5).}
This experiment uses the box $\Sset=[-1,1]^2$ and a known conditional mean
$\mu(x)=2+0.25x$ for $x\sim\mathrm{Unif}([-1,1]^2)$. Conditional costs are
$C=\mu(X)+\varepsilon$, where the coordinates of $\varepsilon$ are independent
uniform draws on $[-a,a]$. We use $a=3$ (SC-on), $a=2.25$ (low-crossing), and
$a=0.25$ (SC-off). Each law remains conditionally centrally symmetric. In the
SC-off design, $C$ lies in the positive orthant, so the positive displacement
$(0.5,0)$ is a flat direction of the selected box decision rule. For this same
displacement, held-out crossing masses are $0.152$, $0.056$, and $0$,
respectively. These finite-sample diagnostics illustrate the designed regimes;
they do not estimate a real-data population assumption.

We train Linear SPO+, MLP-MSE, MLP-SPO+, and MLP-SPO+ with a squared-report
term ($\lambda=0.1$) on five independent seeds. Table~\ref{tab:controlled-neural}
reports held-out Monte Carlo surrogate excess relative to the known $\mu$,
decision regret, quotient-report error $\E\|P_L(\hat c(X)-\mu(X))\|_2$, and
report variance across fitted seed predictions. Excess, regret, and Q-error
are five-seed means with sample SD. Small signed excess estimates
with magnitude below $5\times10^{-9}$ are displayed as zero at table precision.
In SC-off, all four methods obtain zero observed decision regret while their
quotient errors and report variances differ. This is controlled evidence about
equality geometry, not a finite-sample neural-convergence result.

\begin{table*}[t]
\centering
\caption{Controlled neural identification experiment. ``Excess'', ``Regret'',
and ``Q-error'' are five-seed means $\pm$ sample SD; ``Excess'' is held-out
SPO+ excess relative to the known conditional mean and ``Q-error'' is
quotient-report error. ``Var'' is mean coordinatewise fitted report variance
across seeds. The SC-off rows have equal observed decision regret but different
fitted reports.}
\label{tab:controlled-neural}
\small
\begin{tabular}{llrrrr}
\hline
Regime & Method & Excess & Regret & Q-error & Var \\
\hline
SC-on & Linear SPO+ & $0.0146\pm0.0064$ & $0.3326\pm0.0000$ & $0.1217\pm0.0330$ & $0.0095$ \\
      & MLP-MSE & $0.0315\pm0.0159$ & $0.3326\pm0.0000$ & $0.1863\pm0.0437$ & $0.0261$ \\
      & MLP-SPO+ & $0.0400\pm0.0132$ & $0.3326\pm0.0000$ & $0.2046\pm0.0422$ & $0.0290$ \\
      & MLP-SPO+ + MSE & $0.0388\pm0.0119$ & $0.3326\pm0.0000$ & $0.2039\pm0.0446$ & $0.0295$ \\
Low mass & Linear SPO+ & $0.0040\pm0.0018$ & $0.0360\pm0.0000$ & $0.0671\pm0.0194$ & $0.0026$ \\
        & MLP-MSE & $0.0205\pm0.0051$ & $0.0360\pm0.0000$ & $0.1481\pm0.0262$ & $0.0163$ \\
        & MLP-SPO+ & $0.0111\pm0.0035$ & $0.0360\pm0.0000$ & $0.1164\pm0.0323$ & $0.0080$ \\
        & MLP-SPO+ + MSE & $0.0111\pm0.0034$ & $0.0360\pm0.0000$ & $0.1135\pm0.0249$ & $0.0067$ \\
SC-off & Linear SPO+ & $0.0000\pm0.0000$ & $0.0000\pm0.0000$ & $1.2192\pm0.0027$ & $0.0001$ \\
       & MLP-MSE & $0.0000\pm0.0000$ & $0.0000\pm0.0000$ & $0.0305\pm0.0118$ & $0.0007$ \\
       & MLP-SPO+ & $0.0000\pm0.0000$ & $0.0000\pm0.0000$ & $0.3384\pm0.1645$ & $0.0601$ \\
       & MLP-SPO+ + MSE & $0.0000\pm0.0000$ & $0.0000\pm0.0000$ & $0.0742\pm0.0271$ & $0.0019$ \\
\hline
\end{tabular}
\end{table*}

\begin{table}[t]
\centering
\caption{Fitted local log-log slopes (theory in parentheses).
\textsc{Box} and \textsc{Tri} are the polytopes of Exp.~4; slopes are
averaged over the crossing directions; the flat-density row reports a
finite-grid slope from Exp.~2 and its asymptotic behavior.}
\label{tab:mechanism-slopes}
\small
\begin{tabular}{lccc}
\hline
Case & $P(B_{tu})$ & $\E[G_{tu}\mid B_{tu}]$ & Rdiff \\
\hline
Box $\Sset=[-1,1]^2$ (5 directions) & $1.00\ (1)$ & $1.00\ (1)$ & $2.01\ (2)$ \\
\textsc{Tri} (4 directions) & $0.98$ to $1.01\ (1)$ & $1.00\ (1)$ & $1.93$ to $2.07\ (2)$ \\
Lower density floor & $0.99\ (1)$ & $0.99\ (1)$ & $2.00\ (2)$ \\
Flat density $e^{-1/|c|}$ (Exp.~2) & \multicolumn{3}{c}{finite-grid slope $175.30$ at $t\approx0.002$; diverges as $t\downarrow0$} \\
\hline
\end{tabular}
\end{table}

\begin{figure*}[t]
  \centering
  \includegraphics[width=0.98\textwidth]{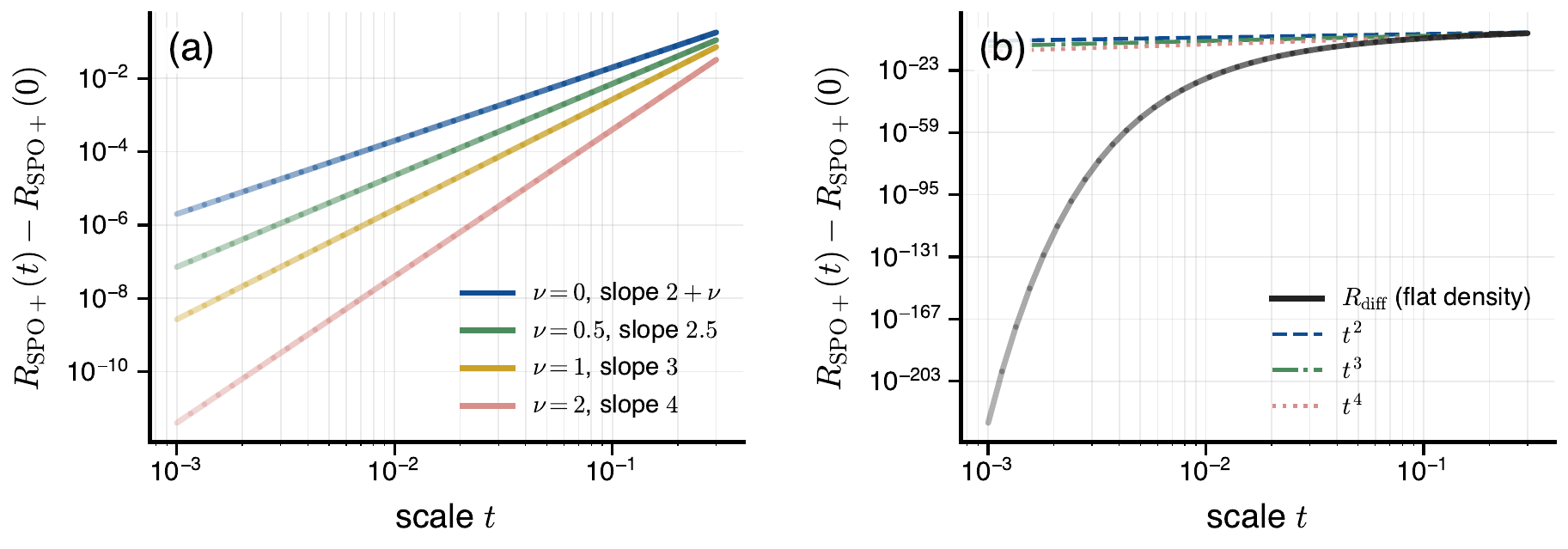}
  \caption{(a) Exact excess risk of the one-dimensional model for
  $\nu\in\{0,0.5,1,2\}$ (log-log; slopes $2+\nu$).  (b) A flat
  $C^\infty$ density satisfies strict crossing but yields no polynomial
  modulus: the excess risk falls below $t^2,t^3,t^4$.}
  \label{fig:mech-A}
\end{figure*}

\begin{figure*}[t]
  \centering
  \includegraphics[width=0.98\textwidth]{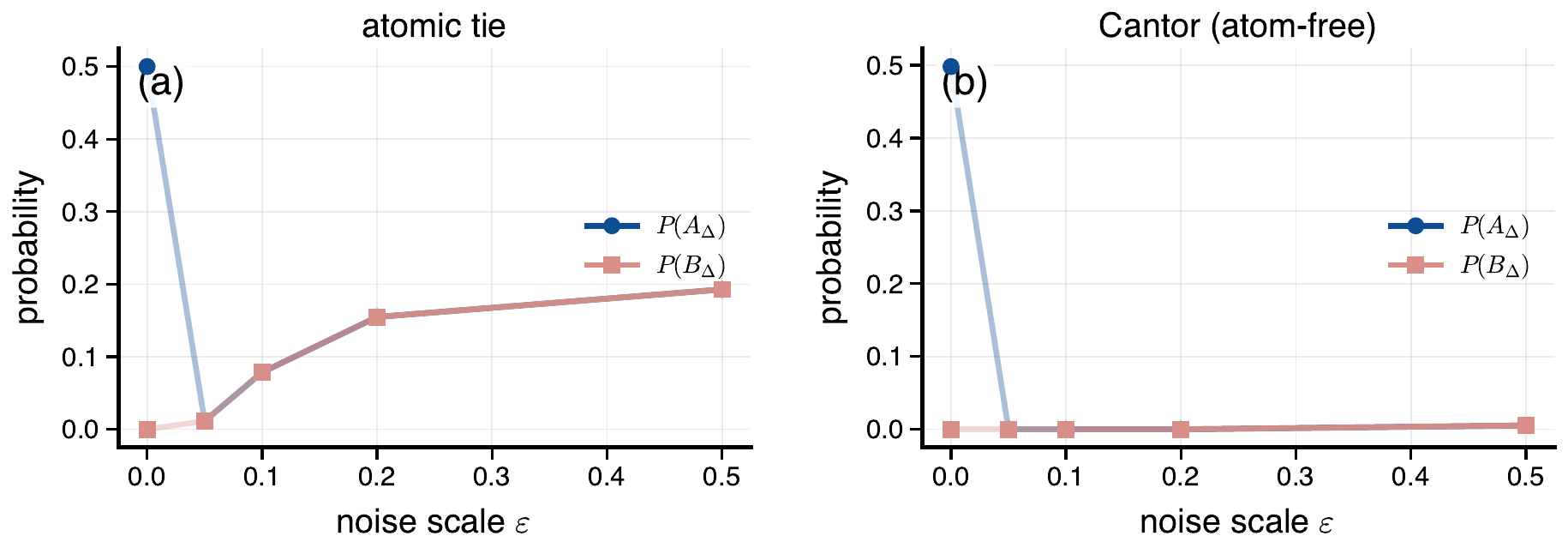}
  \caption{Selected-oracle disagreement $\Pr(A_\Delta)$ versus
  optimal-face departure $\Pr(B_\Delta)$ for the atomic and Cantor examples
  under noise $c_\varepsilon=c+\varepsilon Z$.  At $\varepsilon=0$,
  $\Pr(A_\Delta)=1/2$ and $\Pr(B_\Delta)=0$; for $\varepsilon>0$ they
  coincide (Theorem~\ref{thm:ac-equivalence}).}
  \label{fig:mech-B}
\end{figure*}

\begin{figure*}[t]
  \centering
  \includegraphics[width=0.99\textwidth]{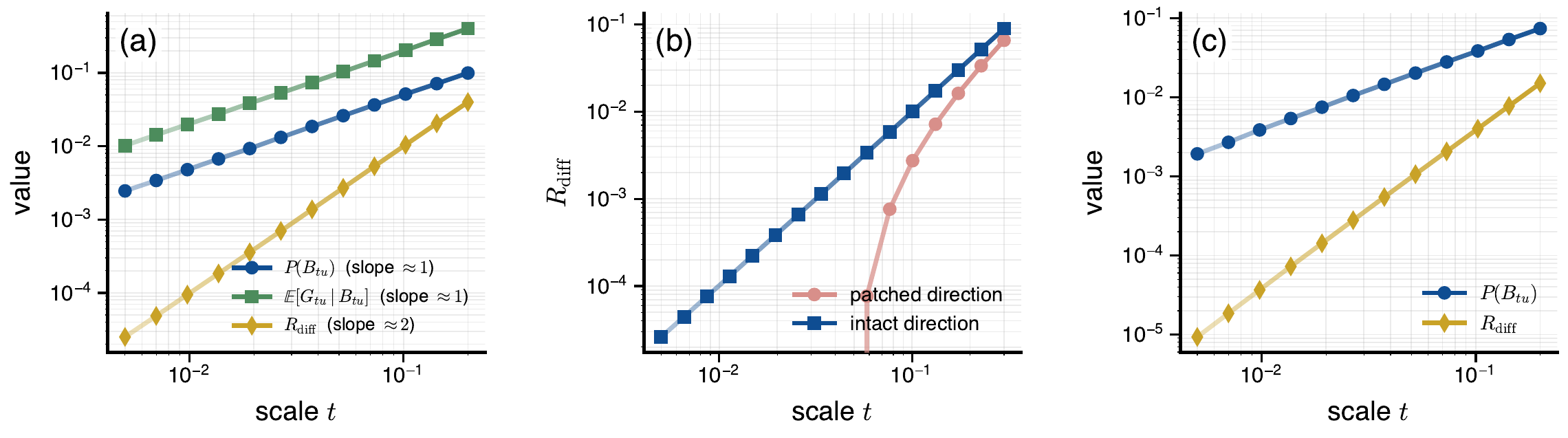}
  \caption{(a) Crossing mass, conditional gap, and excess risk versus scale
  for the box polytope (slopes $1/1/2$).  (b) Removing the boundary patch
  near one facet kills growth below the gap (dead zone), while the intact
  direction is unchanged.  (c) The same mechanism on a non-symmetric
  triangle.}
  \label{fig:mech-C}
\end{figure*}

\end{document}